\documentclass[11pt]{article}
\usepackage[margin=1in]{geometry}
\usepackage{amsmath,amssymb,amsthm}
\usepackage{graphicx}
\usepackage{wrapfig}
\usepackage{booktabs}
\usepackage{multirow}
\usepackage{enumitem}
\usepackage[round]{natbib}
\usepackage{setspace}
\usepackage{dsfont}
\usepackage{xcolor}
\usepackage{titlesec}
\usepackage[hypertexnames=false,linktocpage=true,pagebackref=true,colorlinks,allcolors=blue,bookmarks=true,bookmarksopen,bookmarksnumbered]{hyperref}

\newtheorem{theorem}{Theorem}[section]

\newtheorem{lemma}[theorem]{Lemma}
\newtheorem{claim}[theorem]{Claim}

\newcommand{\argmax}{\operatorname*{arg\,max}}
\newcommand{\eps}{\epsilon}
\newcommand{\boldtau}{\tau}

\newcommand{\Rev}{\mathcal R}

\newcommand{\PublicEdit}[1]{#1}

\titlespacing*{\paragraph}{0pt}
  {2.6ex plus 0.8ex minus 0.16ex}
  {1em}

\allowdisplaybreaks[4]

\title{AI-Mediation in Retail Markets:\\Revenue Effects and Assortment Optimization}
\author{%
Omar El Housni \quad\qquad Marouane Ibn Brahim\\[0.3em]
\normalsize Cornell Tech, Cornell University\\
\normalsize\texttt{\char`\{oe46,mi262\char`\}@cornell.edu}%
}
\date{\today}

\begin{document}
\hypersetup{pageanchor=false}
\maketitle

\begin{abstract}
\PublicEdit{AI shopping assistants increasingly influence which products consumers consider. We study retail assortment planning under AI mediation, where a retailer chooses a catalog, an AI intermediary forms a customer-facing shortlist, and the customer chooses from that shortlist according to a multinomial logit model. We consider both deterministic mediation, in which the AI follows a fixed segment-specific ranking, and stochastic mediation based on the Plackett--Luce model.

Empirically, we construct a dataset of product rankings from leading AI models, estimate Plackett--Luce ranking models, and evaluate their predictive performance on held-out assortments. Combining these estimates with synthetic customer preferences, we find that accounting for AI mediation can materially change optimal assortments and revenue. Across the tested capacity configurations, ignoring mediation yields mean revenue losses between 2.52\% and 11.37\% relative to the PL optimum. A much simpler deterministic plug-in, based only on the ordering of fitted AI utilities, performs substantially better, with mean losses of at most 0.53\%.

Motivated by this performance, we study assortment optimization under deterministic mediation. We characterize its computational complexity and show that, for any fixed number of customer segments, the problem admits a fully polynomial-time approximation scheme.}
\end{abstract}

\hypersetup{pageanchor=true}
\section{Introduction}\label{sec:introduction}
Assortment optimization is a central problem in retailing and e-commerce, where the goal is to determine which products a retailer should make available to customers in order to maximize expected revenue or other business objectives, subject to operational, capacity, or display constraints. A large literature has developed models and algorithms that translate customer-choice behavior into concrete prescriptions for the retailer's assortment decision. A common feature of this classical paradigm, however, is that the retailer directly controls the set of products from which the customer ultimately chooses.
This direct-control assumption is challenged when customers turn to generative-AI systems such as ChatGPT, Claude, and Gemini to streamline their shopping decisions. Recent consumer surveys point to rapidly growing use of these systems for product discovery, comparison, and recommendation. For example, \citet{capgemini2025consumer} reports that \(58\%\) of consumers surveyed across 12 countries use generative-AI tools as a source of product and service recommendations, compared with \(25\%\) in 2023. %

When an AI intermediary stands between the retailer and the customer, the retailer no longer directly controls the customer-facing assortment. Instead, the retailer may expose a larger machine-facing catalog, from which the AI filters or ranks a smaller set of products for the customer. Because the AI's ranking criterion need not coincide with the retailer's revenue objective, this additional layer may change the retailer's optimal assortment decision. It raises both a descriptive and a prescriptive question: \begin{quote} \emph{What are the revenue consequences of AI-mediated shopping for retailers? How should retailers adapt their assortment decisions when an AI intermediary determines which products customers see?}\end{quote}

In this paper, we develop a framework for assortment planning when product exposure is mediated by an AI system, with the goal of understanding how this additional layer affects retailer revenue and optimal assortment decisions. Empirically, we collect controlled ranking data from several leading AI models across different shopper profiles, use these data to estimate models of AI ranking behavior, and combine the fitted models with controlled customer-preference scenarios to study the economic consequences of AI mediation. \PublicEdit{Theoretically, we analyze assortment optimization under deterministic mediation, characterize its computational complexity across different parameter regimes, and develop exact and approximation algorithms where possible.}

\paragraph{Informal model description.} We introduce a stylized model of AI-mediated assortment optimization that isolates this additional decision layer and allows us to study both its revenue consequences and the resulting assortment problem. A retailer has access to a universe of substitutable products, each generating a product-specific revenue when purchased. The retailer first selects a feasible machine-facing catalog, without observing the arriving customer's segment. An AI intermediary then maps this catalog into a smaller customer-facing set according to a known, possibly stochastic and segment-dependent recommendation rule. The customer observes only this AI-generated shortlist and chooses either one of its products or an outside option representing no purchase. If a product is purchased, the retailer receives the corresponding revenue. Thus, unlike in classical assortment optimization, the retailer controls the set available to the intermediary but not necessarily the set from which the customer ultimately chooses. \PublicEdit{Figure~\ref{fig:ai-mediation-flow} illustrates this decision sequence.} Section~\ref{sec:model} formalizes this decision sequence, the recommendation kernels considered in the paper, and the customer's choice model. In a nutshell, we study two stylized models of how the AI selects which products to show. Under \emph{deterministic mediation}, each AI system and customer profile induces a fixed ranking of products, and the AI displays the highest-ranked products among those offered by the retailer, up to its display capacity. Under \emph{PL-stochastic mediation}, the ranking itself is random and follows the classical Plackett--Luce (PL) model \citep{luce1959individual,plackett1975analysis}, with the AI again displaying the highest-ranked products in the realized ranking. To connect these models to actual AI behavior, we conduct a controlled study in which we query leading AI models as shopping assistants for customers with different profiles, and use the resulting ranking data to calibrate and evaluate the models. \PublicEdit{We use both mediation models to study how AI-mediated product exposure affects retailer revenue, and develop the algorithmic and theoretical analysis of the retailer's assortment-planning problem under deterministic mediation.}

\begin{figure}[!t]
    \centering
    \includegraphics[
        width=\linewidth,
        trim=0 0 0 0,
        clip
    ]{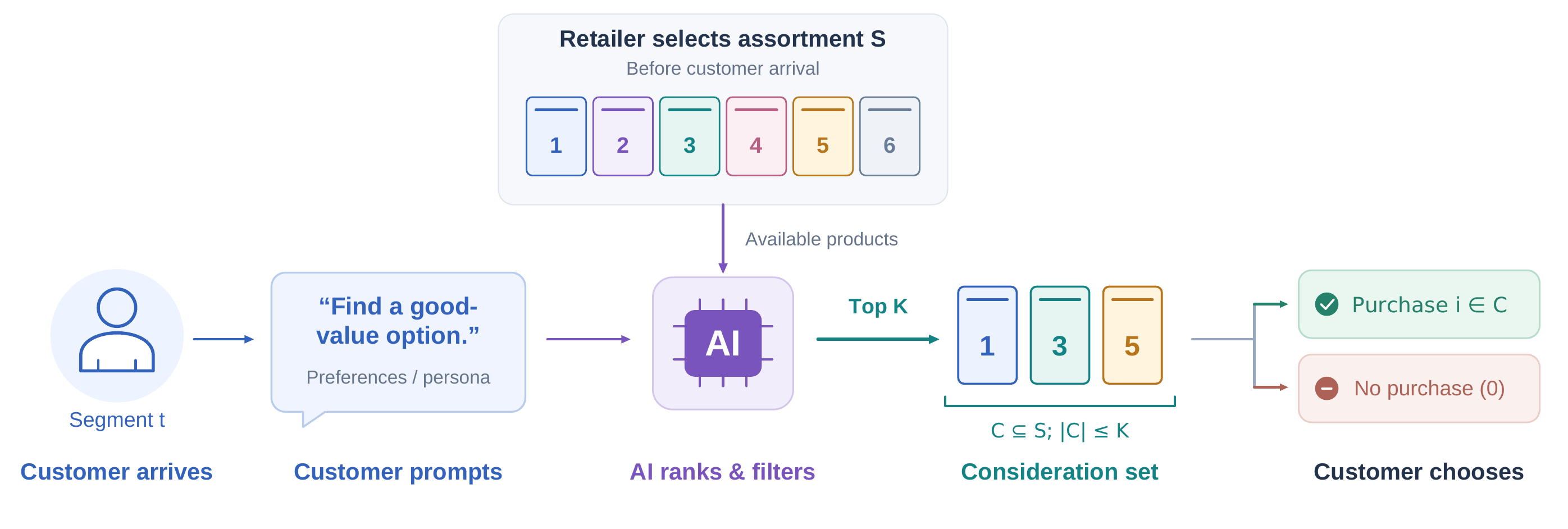}
    \caption{\PublicEdit{AI-mediated shopping sequence.}}
    \label{fig:ai-mediation-flow}
\end{figure}

\subsection{Main Contributions and Technical Ideas}\label{sec:contributions}
\PublicEdit{Our contributions combine an empirical study of AI ranking behavior, an analysis of how AI mediation affects retailer revenue and assortment decisions, and complexity results and approximation algorithms for assortment optimization under deterministic mediation.}
We first construct a controlled dataset of AI recommendations by prompting several AI models, and use these data to estimate and validate Plackett--Luce (PL) models of their ranking behavior. Using the fitted models in semi-synthetic experiments, we find that accounting for mediation leads retailers to choose smaller assortments on average, sometimes avoiding AI filtering altogether.
Ignoring mediation can produce substantial revenue losses. To assess the value of modeling randomness in AI recommendations, we also consider a deterministic approximation: we rank products by their fitted AI utilities and optimize the assortment under this fixed ordering. When evaluated under the stochastic PL model, the resulting assortments recover nearly all of the optimal PL-mediated revenue on average in our experiments. We also show how display capacity and customer–AI preference alignment affect these conclusions.
\PublicEdit{On the theoretical side, we characterize the computational complexity of assortment optimization under deterministic mediation and develop a fully polynomial-time approximation scheme for any fixed number of customer segments.}

\vspace{0.5em}
\noindent \textbf{Empirical characterization and estimation of AI rankings.}
In Section~\ref{sec:ai-ranking-empirics}, we study whether a tractable ranking model can describe the recommendations of AI shopping assistants. Unlike a latent consideration process inferred from purchases, the AI's recommendations can be observed directly by querying AI models with chosen product sets and customer profiles. We use this access to construct ranking data, estimate the mediation layer, and evaluate how well the fitted model predicts AI rankings for assortments not used in estimation.

\begin{enumerate}[label=(\alph*),leftmargin=*,itemsep=0.5em]

\item {\em Ranking data construction and model fitting.}
We generate \(32{,}400\) ranked recommendation lists by prompting models from OpenAI, Anthropic, and Google with real product listings from three categories: Lightning chargers, robot vacuums, and pet leashes. Each model is queried under three customer personas: budget-sensitive, value-seeking, and premium-seeking, which are explicitly specified in the prompts. Each query asks the AI to rank five products from a retailer assortment of ten. We estimate a separate PL model for each category and provider--persona pair, using a ridge-regularized maximum-likelihood approach. We show that the resulting regularized log-likelihood is strictly concave and therefore admits a unique finite maximizer, which can be computed efficiently.

\item {\em Out-of-sample validation.}
We evaluate the fitted PL models on \(5{,}400\) held-out AI rankings whose underlying offer sets were not used for estimation. For each held-out offer set, we use the estimated PL utilities to rank the ten offered products and compare this predicted ranking with the recommendation returned by the AI. The product with the highest estimated utility coincides with the AI's first-ranked product in \(78.8\%\) of observations. The average negative log-likelihood is \(3.312\) per held-out ranking, compared with \(3.360\) on the training sample, indicating essentially no deterioration on previously unseen offer sets. These results suggest that the fitted PL models capture a substantial part of the ranking behavior observed in the AI recommendations.
The fitted utilities also provide a natural deterministic approximation: for each segment, we simply rank products from highest to lowest estimated utility. In Section~\ref{sec:numerics}, we study whether assortments optimized using this simpler fixed ranking continue to perform well when evaluated under the stochastic PL model.
\end{enumerate}

\noindent  \textbf{Experimental effects of mediation on assortments and revenue.}
Our main economic finding is that ignoring mediation can be costly. Surprisingly, accounting only for the ordering of AI preferences, rather than the exact PL fitted utilities, captures most of the attainable revenue in our experiments. Section~\ref{sec:numerics} develops this conclusion by combining the estimated AI-ranking models with synthetic customer-side MNL preferences. We compare three policies: optimization under stochastic PL mediation, optimization under the induced deterministic rankings, and a direct-MNL benchmark that ignores the mediation layer. All policies are evaluated under the same fitted PL mediation model. Throughout, we consider a cardinality-constrained setting in which the retailer may offer at most \(B\) products, and the AI displays at most \(K\) products to the customer, where \(B\) and \(K\) are exogenous positive integers.

\begin{enumerate}[label=(\alph*),leftmargin=*,itemsep=0.5em]

\item {\em AI mediation leads to lower-cardinality assortments.}
When the retailer accounts for AI mediation, it generally chooses to offer fewer products than when it ignores the AI. Across the tested configurations, the direct-MNL assortment contains \(1.48\) to \(4.15\) more products on average than the PL-optimal assortment.
There are two ways in which the retailer responds to the AI's display limit \(K\). In some instances, the retailer offers more than \(K\) products and allows the AI to choose which \(K\) products the customer sees. In other instances, the retailer instead reduces its assortment to \(K\) products or fewer, so that every product it offers is shown to the customer.

\item {\em Ignoring mediation creates substantial losses, especially when display capacity is tight.}
The direct-MNL policy incurs mean revenue losses ranging from \(2.52\%\) to \(11.37\%\) across the considered parameter configurations, relative to the PL-optimal assortment.
The losses are not confined to averages: the configuration-specific \(95\)th percentiles range from \(10.36\%\) to \(30.66\%\), and the largest individual loss is \(57.08\%\).
Moreover, losses tend to be larger when the AI displays fewer products to the customer.

\item {\em A deterministic plug-in recovers nearly all optimal revenue on average.}
We next ask how much is lost if the retailer accounts for the AI's ranking of products but ignores the randomness in those rankings. The deterministic model keeps the same segment-specific product ordering and customer-choice model, but assumes that the AI always follows a fixed ranking. For each instance, we optimize the assortment under this deterministic model and then evaluate the resulting assortment under the fitted stochastic PL model. Across the tested configurations, the mean revenue loss relative to the PL optimum ranges from only \(0.08\%\) to \(0.53\%\). These losses are far smaller than those obtained when the retailer ignores mediation altogether, suggesting that most of the value of accounting for mediation comes from capturing the AI's average ranking of products rather than the full stochastic variation in its recommendations.

\item {\em Preference alignment reduces the cost of ignoring mediation.}
We vary how strongly synthetic customer preferences align with a persona-level summary of the fitted AI preferences, using an alignment parameter \(\rho\in[0,1]\), where a larger \(\rho\) represents higher alignment. Deferring the formal definition of alignment to Section~\ref{sec:numerics}, this parameter controls the portion of the customer's utility which is aligned with the AI and a complementary random portion capturing each customer's idiosyncrasies. Mean direct-MNL revenue losses range from \(5.75\%\) to \(24.06\%\) when \(\rho=0\), compared with \(0.59\%\)--\(5.74\%\) when \(\rho=1\), although the intervening relationship is not always monotone.

\item {\em Greater alignment can make retailers more willing to allow filtering.}
The lower cost of ignoring mediation at higher alignment does not mean that filtering disappears. At \(\rho=0\), active filtering occurs in at most \(8.9\%\) of instances across capacity configurations; at \(\rho=0.95\), its share ranges from \(44.4\%\) to \(66.7\%\). Thus, greater alignment can make it optimal to leave more products for the AI to filter.

\end{enumerate}

\noindent \textbf{Computational complexity and approximation under deterministic mediation.}
The economic performance of the deterministic plug-in motivates studying the corresponding assortment optimization problem. In particular, we consider the cardinality-constrained setting in which the retailer may offer at most \(B\ge 1\) products. In Section~\ref{sec:deterministic}, we characterize its computational complexity in terms of the number of segments \(m\) and display capacity \(K\). We then construct an approximation scheme for any arbitrary precision. Here, \(n\) denotes the number of products.

\begin{enumerate}[label=(\alph*),leftmargin=*,itemsep=0.5em]

\item {\em Complexity landscape under deterministic mediation.}
When the number of segments \(m\) is unrestricted, deterministic mediated assortment optimization is NP-hard even for \(K=1\) and \(B=n\). Moreover, for every fixed \(\epsilon\in(0,1)\), the problem admits no polynomial-time \(O(n^{1-\epsilon})\)-approximation unless \(\mathrm{P}=\mathrm{NP}\). Our reduction is from assortment planning under ranking preferences \citep{aouad2018approximability}. At the same time, the problem becomes tractable in several important parameter regimes. In particular, we show that there exists an optimal assortment containing at most \(mK\) products, since we can safely remove any product that is never displayed to any segment. Consequently, when both \(m\) and \(K\) are fixed, enumerating these bounded-size assortments gives a polynomial-time exact algorithm. With a single segment, the problem is also polynomial-time solvable for arbitrary \(K\), because it reduces to a cardinality-constrained MNL assortment problem. In contrast, fixing \(m\ge 2\) alone does not recover exact polynomial-time solvability: taking \(K=B=n\) recovers assortment optimization under a mixture of MNL models, which is weakly NP-hard already with two segments \citep{rusmevichientong2014assortment}. The resulting complexity landscape is summarized in Table~\ref{tab:det-computational-landscape}.%

\item {\em An FPTAS for a fixed number of segments.}
Our main algorithmic contribution for the deterministic mediated assortment problem  is an approximation scheme that, for any \(\epsilon\in(0,1)\), returns a feasible assortment whose revenue is at least \(1-\epsilon\) times optimal. For every fixed \(m\), the running time is polynomial in the binary encoding length of the instance and \(1/\epsilon\), yielding a fully polynomial-time approximation scheme (FPTAS).
The key idea is to enumerate the worst-ranked displayed product for each segment, which we refer to as threshold products. Fixing these thresholds determines which selected products contribute to each segment's displayed set, provided that exactly \(K\) selected products lie at or above its threshold. Conditional on these thresholds, the objective is a sum of \(m\) ratios of affine functions. We approximate this objective using a framework for sums of ratios, whose exact-solvability condition is verified through multidimensional knapsack.
\end{enumerate}

\subsection{Related Literature}\label{sec:related-literature}
\paragraph{Assortment optimization.}
Assortment optimization studies how a retailer should select a subset of products from a larger universe when customers substitute among the products made available to them; see \citet{Kök2009} for a broad review. In their seminal paper, \citet{talluri2004revenue} formulate choice-based revenue management under a general discrete-choice model and identify the revenue-ordered structure of optimal assortments under the multinomial logit model. A large subsequent literature develops structural and algorithmic results under richer models of substitution, including the nested logit model \citep{gallego2014constrained,feldman2015capacity} and the Markov-chain choice model \citep{desir2020constrained,desir2024robust}. Other work studies operational or distributional restrictions on the retailer's decision, including cardinality and capacity constraints \citep{desir2022capacitated}, minimum-visibility requirements \citep{barre2026assortment}, fairness or exposure constraints \citep{lu2023step, chen2022fair} and covering constraints \citep{housni2024assortment}. Although these papers differ substantially in their choice models, objectives, and feasible assortment families, the core decision sequence is generally the same: the retailer chooses the customer-facing assortment, and the customer chooses directly from that set. Our model inserts an additional decision layer. The retailer chooses a machine-facing catalog \(S\), while the customer-facing set \(C\subseteq S\) is generated by an AI intermediary. The retailer therefore controls the input to the mediation rule rather than the set from which the customer ultimately chooses.

\paragraph{Consider-then-choose models.} The closest line of research to our paper is the literature on consideration sets in customer choice. Marketing models have long recognized that consumers may first screen a large set of alternatives and then choose only among those that survive the screening stage. Early work explains consideration-set formation through evaluation costs and the incremental value of examining additional alternatives \citep{hauser1990evaluation,roberts1991development}. In a stochastic-choice formulation, \citet{manzini2014stochastic} assume that each alternative is independently considered with an alternative-specific attention probability before the consumer selects the most preferred considered alternative.

Within the operations literature, \citet{aouad2021assortment} develop a general dynamic-programming framework for assortment optimization under structured consider-then-choose models, in which products are first screened and then ranked. The closest mathematical connection to our general formulation is the click-based MNL model of \citet{aouad2025click}. In that model, a customer considers only the products that she clicks and subsequently chooses among the clicked products according to an MNL model. The seller selects the set of products made available for purchase, and the authors establish a polynomial-time approximation scheme for the resulting assortment problem.
\citet{gallego2024random} operationalize this random-consideration-set model for demand estimation, assortment optimization, and pricing.

In our paper, we consider a different class of first-stage mediation mechanisms. As formalized in Section~\ref{sec:model}, we study two ranking-based models of AI mediation. In the first, each customer segment is associated with a deterministic ranking of the products, and the AI displays the top \(K\) products from the retailer's machine-facing catalog according to this ranking. In the second, the AI's ranking is stochastic and follows the classical Plackett--Luce model \citep{plackett1975analysis,luce1959individual}: each product is assigned a positive weight, a random ranking is generated according to these weights, and the AI displays the top \(K\) products in that ranking. \PublicEdit{To the best of our knowledge, these ranking-based mediation mechanisms, and the resulting assortment-optimization problems, are not captured by the consideration-set models studied in the existing assortment-optimization literature.} 
A key difference in our setting is that the consideration set is generated by an AI system rather than formed internally by the consumer. In traditional consider-then-choose models, the consumer's consideration set is typically latent and must be inferred indirectly from observed choices or auxiliary data. Here, while the consideration sets encountered in actual shopping interactions are still not directly observed by the retailer, the underlying AI systems can be queried in controlled settings with chosen product sets and customer profiles. This makes it possible to generate targeted ranking data, fit tractable and interpretable models of AI behavior, and then use those models for revenue analysis and assortment optimization.

\paragraph{AI-shopping agents.}
A growing literature studies how AI systems alter product discovery, recommendation, and purchasing. \citet{allouah2026agent} study the behavior of AI shopping agents when they are asked to choose products from a given assortment. Through controlled experiments across several AI models, they document systematic responses to product and marketplace attributes, including prices, ratings, reviews, listing positions, and platform endorsements, as well as substantial heterogeneity across AI models in the choices they induce. Their setting treats the AI agent as the final decision maker, and the main objective is descriptive: to understand and model how the AI chooses among available products. Our paper instead studies a different decision process in which the AI first filters the retailer's assortment and the customer then makes the final purchase decision from the resulting recommendation set. This additional layer creates a prescriptive problem for the retailer: how should the assortment be chosen when the retailer controls the products made available to the AI, but the AI controls which of those products the customer ultimately sees? Our focus is therefore on the retailer's assortment-planning problem and on the revenue implications of this additional mediation layer. At the retailer-controlled end of the spectrum, \citet{liu2026governed} develop an inventory-aware agentic policy for basket recommendation that jointly accounts for personalization, product availability, and operational constraints. In that setting, the recommendation agent is itself a decision instrument of the retailer, whereas we treat the AI intermediary as a separate behavioral layer whose ranking objective need not coincide with the retailer's revenue objective.

Several recent papers instead study how a shopping agent should interact with or act on behalf of a consumer. \citet{cao2026solicit}
introduce a solicit-then-suggest model in which an agent asks a sequence of preference-elicitation questions and then recommends \(k\) products. They characterize the trade-off between solicitation depth and recommendation-set breadth, as well as the structure of the agent's optimal elicitation and recommendation policy. \citet{dong2025humanai}
study costly and noisy communication between a customer and a recommendation agent, jointly optimizing the information communicated by the customer and the diversity of the agent's recommendation set. \citet{fu2026strategic} study a buying agent that monitors prices over time and determines when to purchase on the consumer's behalf. These papers optimize the agent's information acquisition, recommendation, or purchasing policy from the consumer's perspective. We instead study the higher-level prescriptive problem faced by the retailer: taking the intermediary's behavior as given or estimated, how should the retailer choose its assortment when the intermediary, rather than the retailer, determines which products are ultimately exposed to customers?

\paragraph{Random-utility, Plackett-Luce, and multinomial logit models.} Our ranking and customer choice specifications build on the random-utility literature. In a random-utility model, each alternative receives a deterministic utility component and a random shock, and the selected alternative is the one with the highest realized utility \citep{mcfadden1974conditional}. When the shocks are independent standard Gumbel-distributed random variables, the probability that an alternative attains the highest utility has the multinomial-logit form associated with Luce's choice axiom and conditional logit \citep{luce1959individual,mcfadden1974conditional}. The MNL model is widely used in assortment optimization because its choice probabilities admit a parsimonious closed form, although its independence-of-irrelevant-alternatives structure can be restrictive.
The Plackett--Luce model extends the same probability structure from a single choice to an ordered ranking \citep{luce1959individual,plackett1975analysis}. It can be generated by sorting alternatives according to independent Gumbel-perturbed utilities. The relationship between the two models is therefore probabilistic rather than terminological: MNL describes the top-ranked alternative, whereas Plackett--Luce describes an ordered sequence of alternatives.

We use this family of models at two distinct stages. On the AI side, segment-\(t\) Plackett--Luce weights \(w_{it}\) govern a random ranking of the products in the seller's catalog, and the AI displays the products occupying the first \(K\) positions. On the customer side, separate MNL weights \(v_{it}\), together with an outside option, govern the customer's purchase decision from the displayed shortlist. These are different behavioral primitives: the AI-side weights determine exposure, whereas the customer-side weights determine substitution and purchase conditional on exposure. We impose no requirement that the two sets of utilities coincide, and study the effect of alignment in Section~\ref{sec:numerics}.

\PublicEdit{The remainder of the paper is organized as follows. Section~\ref*{sec:model} presents the model; Section~\ref*{sec:ai-ranking-empirics} describes the AI-ranking data, estimation, and validation; Section~\ref*{sec:numerics} studies the economic effects of mediation through numerical experiments; Section~\ref*{sec:deterministic} develops the deterministic complexity and approximation results.}

\section{Model Formulation}
\label{sec:model}

This section formulates the AI-mediation model used throughout our work, by modelling the interaction between the retailer and the customer through an AI intermediary. The central distinction from classical assortment optimization is that the retailer does not directly choose the set from which customers make purchase decisions.
Instead, the retailer first selects a machine-facing catalog, then an AI intermediary maps this catalog into a smaller customer-facing set, and finally, the customer chooses from the mediated set. We define this process through a general recommendation kernel in Section~\ref{subsec:general-model}, and  specialize our model to ranking-based kernels in Section~\ref{subsec:ranking-kernels}.

\subsection{General Model of AI-Mediation}
\label{subsec:general-model}
\paragraph{Products and customer segments.}
Let \(N=\{1,\ldots,n\}\) denote a universe of substitutable products, with each product \(i\in N\) generating revenue \(r_i>0\) when purchased. Customers belong to one of \(m\) segments, indexed by \(t\in\mathcal T=\{1,\ldots,m\}\). Segment \(t\) occurs with probability \(\lambda_t>0\), with \(\sum_{t\in\mathcal T}\lambda_t=1\).
We use segments to capture heterogeneity in both customer preferences and the behavior of the AI assistant through which the customer shops. Concretely, a segment can be interpreted as a particular subpopulation of customers who use a particular AI system. Both components (subpopulation and AI model) matter: two customers with the same underlying preferences may belong to different segments if they use different AI systems and therefore receive different recommendations, while two customers using the same AI system may belong to different segments if their preferences differ. In our numerical experiments, we make this interpretation explicit by defining segments through combinations of customer profiles, such as budget-sensitive or premium-seeking customers, and AI models, such as ChatGPT or Gemini Flash. Thus, for example, budget-sensitive customers using ChatGPT constitute a different segment from budget-sensitive customers using Gemini Flash, and also from premium-seeking customers using ChatGPT. Throughout the paper, however, we keep the notion of a segment fully general and impose no additional structure on \(\mathcal T\) unless stated otherwise.

\paragraph{Seller assortment and AI mediation.}
The retailer chooses an assortment \(S\in\mathcal F\), where \(\mathcal F = \{S\subseteq N\,:\,|S|\leq B\}\) denotes the family of feasible assortments consisting of at most $B$ products, where $B$ is an exogenous positive integer. The retailer selects a single assortment to display to an arriving request, without observing the corresponding segment.
For each segment \(t\in\mathcal T\), we represent the AI's mediation rule by a probability kernel \(Q_t(C\mid S)\) over subsets \(C\subseteq S\).
Specifically, \(Q_t(C\mid S)\) denotes the probability that segment $t$'s AI displays the customer-facing set \(C\) when given assortment \(S\). Importantly, our model allows for the mediation rule to depend on the customer segment. The dependence on \(t\) allows the AI's recommendations to vary across segments, reflecting differences both in the AI system being used and in the customer's profile or preferences.
For every \(S\subseteq N\), this kernel satisfies \(Q_t(C\mid S)\ge 0\) for all \(C\subseteq S\) and \(\sum_{C\subseteq S}Q_t(C\mid S)=1\).

This kernel formulation is intended to serve as a general umbrella framework for AI-mediated assortment optimization. In its unrestricted form, \(Q_t(C\mid S)\) is too general to admit a meaningful structural analysis. We introduce it to formalize the general retailer--AI--customer interaction and to define the assortment optimization problems to follow. In Section~\ref{subsec:ranking-kernels}, we introduce two ranking-based mediation models that constitute the focus of our analysis. Throughout, the retailer takes the mediation rule as given and can affect the customer-facing set only through its choice of the assortments.

\paragraph{Customer preferences.}
We assume that each customer segment follows a segment-specific multinomial logit (MNL) choice model. For every segment \(t\in\mathcal T\), each product \(i\in N\) is associated with a preference weight \(v_{it}\ge 0\), which captures its attractiveness to customers in that segment. We also include an outside option, indexed by \(0\), with a scaled preference weight of $1$, without loss of generality.
Subsequently, when a segment-\(t\) customer is presented with a set of products \(C\subseteq N\), the probability of choosing product \(i\in C\) is
$$
P_t(i\mid C)
=
\frac{v_{it}}
{1+\sum_{j\in C}v_{jt}}.
$$
The probability of choosing the outside option is $
P_t(0\mid C)=1/(1+\sum_{j\in C}v_{jt}).
$

\paragraph{AI-mediated assortment problem.}
Based on the preceding discussion, given a customer-facing set \(C\subseteq N\), the expected revenue generated by a segment-\(t\) customer is

$$
R_t(C)
=
\sum_{i\in C} r_i P_t(i\mid C)
=
\frac{\sum_{i\in C} r_i v_{it}}
{1+\sum_{i\in C} v_{it}}.
$$
Since the AI associated with segment \(t \in {\cal T}\) selects the displayed set \(C\) according to the kernel \(Q_t\), the expected revenue from a segment-\(t\) customer, conditional on retailer assortment \(S\), is$$
    \sum_{C\subseteq S} Q_t(C\mid S)R_t(C).
$$
Thus, the overall segment-weighted expected revenue given a retailer assortment \(S\in\mathcal F\) is
\begin{equation*}
\Rev_Q(S)
=
\sum_{t\in\mathcal T}\lambda_t \cdot
\sum_{C\subseteq S}
Q_t(C\mid S)R_t(C).
\end{equation*}
We introduce the {\em AI-Mediated assortment optimization problem}, denoted by \ref{prob:ai-ao}:
\begin{equation*}
\max_{S\in\mathcal F} \Rev_Q(S).
\tag{{\sf AIM}}
\label{prob:ai-ao}
\end{equation*}

\subsection{Ranking-Based Mediation}
\label{subsec:ranking-kernels}
As previously mentioned, the unrestricted mediation kernel is intended to serve as a general framework and is too broad to study directly. In this paper, we focus on two specific mediation rules, both based on cardinality-constrained rankings. Let
$K \in \{1,\ldots,n\}$ denote the AI's display capacity.
When the AI is given an assortment \(S \subseteq N\), it ranks these products according to a known segment-specific ranking, which may be deterministic or stochastic, and displays the top \(K\) products when \(|S|\ge K\); if \(|S|<K\), the entire assortment is displayed.

Formally, let \(\mathfrak S_N\) denote the set of permutations of \(N\). For any permutation \(\sigma\in\mathfrak S_N\), let
\(
    \operatorname{TopK}_{\sigma}(S) 
\)
denote the \(\min\{K,|S|\}\) highest-ranked products in \(S\) according to
\(\sigma\). Suppose that segment \(t\) is associated with a distribution
\(\mu_t\) over \(\mathfrak S_N\). Then, the induced ranking-based mediation kernel is
\begin{equation}
    Q_t(C\mid S)
    =
    \sum_{\sigma\in\mathfrak S_N}
    \mu_t(\sigma)
    \mathds{1}
    \left\{
        C=\operatorname{TopK}_{\sigma}(S)
    \right\}.
    \label{eq:ranking-kernel}
\end{equation}
In other words, the AI displays the subset $C\subseteq S$ if and only if the products in $C$ represent the top-$K$ according to the drawn ranking \(\sigma\), or if $C$ is the entire assortment when \(|S|<K\).

In this paper, we consider two ranking rules whose specifics are provided in Sections~\ref{subsubsec:model-deterministic} and~\ref{subsubsec:model-PL}. The first is a segment-specific deterministic ranking. The second is based on the classical Plackett--Luce model, in which products are assigned segment-specific positive weights that induce a random ranking, and the AI displays the top \(K\) products under that ranking.

\subsection{Deterministic Rankings}\label{subsubsec:model-deterministic}
Deterministic mediation is obtained when \(\mu_t\) places unit mass on a single ranking \(\sigma_t\). In that case, for any fixed assortment \(S\subseteq N\), the AI displays a single subset \(C_t^{\mathrm D}(S)\) almost surely to each segment-\(t\) customer. Specifically, \(C_t^{\mathrm D}(S)\) consists of the top-\(K\) products according to \(\sigma_t\) when \(|S|\ge K\), and the entire assortment when \(|S|<K\). Hence, \(
    Q_t^{\mathrm D}(C\mid S)
    =
    \mathds{1}
    \{
        C=C_t^{\mathrm D}(S)
    \}.
\)
The expected revenue reduces to
\begin{equation*}
    \Rev^{\mathrm D}(S)
    =
    \sum_{t\in\mathcal T}\lambda_t
    R_t\!\left(C_t^{\mathrm D}(S)\right).
\end{equation*}
Finally, we define {\em AI-Mediated assortment optimization under deterministic mediation} as
\begin{equation}
\max_{S\in\mathcal F} \Rev^{\mathrm D}(S).
\tag{$\text{{\sf AIM}}_\text{\sf D}$}
\label{prob:d-ai-ao}
\end{equation}

Deterministic mediation provides a natural benchmark in which, for each segment, the AI follows a fixed ranking of products and recommends the highest-ranked products available in the retailer's assortment. This formulation is also highly expressive: any stochastic ranking rule of the form \eqref{eq:ranking-kernel} can be represented as a distribution over deterministic rankings, and random-utility models admit such ranking-based representations. Hence, by associating a separate segment with each ranking realization, the deterministic-mediation model can represent any such stochastic ranking model exactly. The only caveat is computational: an exact representation may require exponentially many rankings, and therefore exponentially many segments.

\subsection{Plackett--Luce Rankings}\label{subsubsec:model-PL}
Concurrently, we also consider rankings generated by the classic Plackett--Luce model \citep{luce1959individual,plackett1975analysis}, which is one of the most widespread models considered in theory and practice for ranked-choice data. It is widely used because it is simple, flexible enough to capture differences in product attractiveness, and easy to estimate from observed rankings. It is also closely related to the multinomial logit model: each stage of a Plackett--Luce ranking follows the same logit choice rule. This choice is also in line with recent work on AI shopping behavior; for example, \citet{allouah2026agent} use a conditional logit model to study how AI agents respond to product and marketplace attributes.

\paragraph{The Plackett-Luce model.}
In the PL model, each product receives a utility composed of a deterministic nominal component and a stochastic shock capturing the inherent uncertainty in preferences. Formally, for each product \(i\) and segment \(t\), let \(u_{it}\in\mathbb R\) denote a nominal AI utility. The score of product $i$ for segment $t$'s AI is given by
\(
    U_{it}
    =
    u_{it}+\epsilon_{it},
\)
where the shocks \(\epsilon_{it}\) are independent standard Gumbel random variables. These scores then induce a permutation describing the decreasing order of scores.

\paragraph{Induced mediation kernel.}
We now derive the mediation kernel induced by the Plackett--Luce ranking model. Fix an assortment \(S\subseteq N\), and let \(k(S)=\min\{K,|S|\}\) denote the number of products displayed by the AI. For each segment \(t\), write \(w_{it}=\exp(u_{it})\) for the positive Plackett--Luce weight of product \(i\). These quantities represent AI-side weights and are distinct from the customer-side MNL weights \(v_{it}\), which govern the customer's final choice from the displayed set.

Consider a candidate displayed set \(C\subseteq S\) with \(|C|=k(S)\). The event that the AI displays exactly \(C\) occurs whenever the first \(k(S)\) positions of the realized ranking coincide with the products in \(C\), in any order. Hence,
\begin{equation}\label{eq:pl-kernel}
    Q_t^{\mathrm{PL}}(C\mid S)
=
\sum_{\tau\in\mathfrak S_C}
p_t^{\mathrm{PL}}(\tau\mid S),
\end{equation} 
where \(\mathfrak S_C\) denotes the set of all permutations of \(C\), and \(p_t^{\mathrm{PL}}(\tau\mid S)\) is the probability that the ordered list \(\tau=(\tau_1,\ldots,\tau_{k(S)})\) occupies the first \(k(S)\) positions of the Plackett--Luce ranking.

For any assortment $S\subseteq N$ and any market segment $t \in {\cal T}$, let $W_{t}(S) = \sum_{i\in S} w_{it}$, the total PL weight of assortment $S$ for segment $t$.

In the next lemma, we express the probability of a given ordered list of
products appearing in the first \(k(S)\) positions of the
Plackett--Luce ranking. The result follows directly from the
Gumbel-Top-\(k\) representation of sampling without replacement; see,
e.g., \citet[Theorem~1]{kool2019stochastic}.

\begin{lemma}[Plackett--Luce ranking probabilities]
\label{lem:pl-ranking-probabilities}
For every assortment \(S\) and every ordered list
\(\tau=(\tau_1,\ldots,\tau_{k(S)})\) of distinct products from \(S\),
\begin{equation*}
p_t^{\mathrm{PL}}(\tau\mid S)
=
\prod_{\ell=1}^{k(S)}
\frac{w_{\tau_\ell t}}
{W_t\left(S\setminus\{\tau_1,\ldots,\tau_{\ell-1}\}\right)}.
\end{equation*}
\end{lemma}

Combining the lemma with Equation~\eqref{eq:pl-kernel} gives an explicit expression for the induced mediation kernel \(Q_t^{\mathrm{PL}}(C\mid S)\).
We can now define the resulting retailer's revenue function under Plackett--Luce mediation as
\[
    \Rev^{\mathrm{PL}}(S)
    =
    \sum_{t\in\mathcal T}\lambda_t
    \sum_{\substack{C\subseteq S\colon\\ |C|=k(S)}}
    Q_t^{\mathrm{PL}}(C\mid S)
    R_t(C).
\]
Finally, we define {\em AI-Mediated assortment optimization under Plackett--Luce mediation} as
\begin{equation}
\max_{S\in\mathcal F} \Rev^{\mathrm{PL}}(S).
\tag{$\text{\sf AIM}_{\text{\sf PL}}$}
\label{prob:pl-ai-ao}
\end{equation}

\section{AI-Ranking Data, Estimation, and Validation}
\label{sec:ai-ranking-empirics}

In this section we present our data-generation process, including the product universes examined, prompting conditions, and train--test design. Specifically, we consider a segmented market as described in Section~\ref{subsec:general-model} and generate AI rankings for each segment by prompting the corresponding AI model with randomly sampled product offer sets.
We then estimate a separate Plackett--Luce (PL) model for each market segment. To this end, we express the PL likelihood of the observed AI recommendations and establish concavity of the corresponding log-likelihood. We subsequently use the test samples to evaluate out-of-sample predictive performance.

This estimation step focuses on the stochastic PL model, which allows AI recommendations to vary across queries. In particular, even within a single market segment, the same two products need not be consistently ranked in the same relative order. The deterministic model, by contrast, assigns each segment a single ranking, so fitting it directly to data with such variation would require introducing an additional criterion for reconciling conflicting observed rankings. Instead, we obtain a natural deterministic counterpart from the PL estimates, where products are ranked according to their fitted utilities. In Section~\ref{sec:numerics}, we evaluate this deterministic plug-in and demonstrate that it performs well relative to the stochastic PL benchmark in our experiments.

\subsection{Data Generation and Prompting Protocol}
\label{sec:ranking-data}
\paragraph{Product data and universes.}
In what follows, we utilize the public Amazon Reviews 2023 product metadata dataset \citep{hou2024bridging} to construct three product universes, corresponding to different product categories.
Our three product universes are stored in category-specific CSV files, with one row per product.
Each row records a product identifier, its Amazon product page title, the listed price and brand/store, and a set of category-specific product attributes.
These files describe \(25\) Lightning chargers, \(22\) robot vacuums, and \(25\) pet leashes, for \(72\) products in total, whose characteristics are summarized in Table~\ref{tab:product-universe-summary}. Representative examples are provided in Appendix~\ref{app:pl-calibration}.

\begin{table}[!htbp]
\centering
\setstretch{1}
\caption{Product universes and information available to the AI ranking
models.}
\label{tab:product-universe-summary}

\footnotesize
\setlength{\tabcolsep}{4pt}
\renewcommand{\arraystretch}{1.0}

\begin{tabular}{@{}p{3.4cm} p{2cm} p{5.9cm} p{4.2cm}@{}}
\toprule
{\raggedright Product category\par}
& {\raggedright Number of products\par}
& {\raggedright Product attributes\par}
& {\raggedright Listed price (\$):\\median [min, max]\par} \\
\midrule

{\raggedright Lightning chargers\par}
& 25
& {\raggedright Wattage; connector type; cable length; port count\par}
& {\raggedright 16.79 [6.90, 32.65]\par} \\

{\raggedright Robot vacuums\par}
& 22
& {\raggedright Suction power; run time; Wi-Fi capability; mopping; self-emptying\par}
& {\raggedright 200.86 [49.99, 1,149.99]\par} \\

{\raggedright Pet leashes\par}
& 25
& {\raggedright Product type; species; length; closure type; material\par}
& {\raggedright 18.00 [8.85, 29.95]\par} \\

\bottomrule
\end{tabular}
\end{table}

\paragraph{AI models and prompting personas.}
We consider three widely popular AI models, which were prompted to produce rankings among randomly generated product offer sets:
\begin{itemize}
    \item OpenAI's ChatGPT 5.4 (\texttt{gpt-5.4-mini-2026-03-17}).
    \item Anthropic's Claude Haiku 4.5 (\texttt{claude-haiku-4-5-20251001}).
    \item Google's Gemini 3.6 Flash (\texttt{gemini-3.6-flash}).
\end{itemize}

We also consider a stylized customer population consisting of three sub-populations, which are represented via three prompting personas: 
\begin{itemize}
    \item \emph{Budget-sensitive}: Prioritizes the cheapest available option and permits compromises in features, quality, or brand reputation.
    \item \emph{Value-seeking}: Prioritizes the best price-to-quality tradeoff, accepting a higher price for a meaningful quality improvement but not for a marginal improvement or brand name alone.
    \item \emph{Premium-seeking}: Prioritizes quality, premium features, and reputable brands, with price treated as secondary.
\end{itemize}
Each model is queried under each of these personas, yielding nine
model--persona pairs, which constitute the market segments used in our empirical analysis.

\paragraph{Prompting and AI-data generation.}
We next describe the protocol employed to generate the AI-ranking data. For each persona, we specify a persona-specific instruction describing the shopping preferences that the AI model is asked to represent. This instruction is inserted into a common prompt template that is subsequently instantiated with the product category and the products contained in a given offer set. For example, the persona-specific instruction for value-seeking customers is:
\begin{quote}
``You are assisting a value-seeking customer who cares most about the best
price-to-quality tradeoff. They will pay more for a meaningful step up in
quality, but do not want to overpay for marginal improvements or for brand
name alone.''
\end{quote}

In a nutshell, each AI prompt consists of a persona-specific instruction such as the one above, and a description of the product category. Then, the prompt proceeds by offering a list of {ten products} to the AI model, which is instructed to return a list of {five products} from this offer set, representing a ranked list of its top five preferences.
Each product is represented by its category-specific product number, title, listed price, brand or store, and the category-specific attributes appearing in Table~\ref{tab:product-universe-summary}. In particular, listed price is included in every ranking prompt. The complete prompt template and model request settings are reported in Appendix~\ref{app:pl-calibration}.

For each product category, we sample \(600\) distinct offer sets of cardinality \(10\) uniformly without replacement from the collection of all cardinality-\(10\) subsets of the corresponding product universe. This choice is motivated by \citet{hajek2014minimax}, who show that random uniform product assignment yields near-minimax-optimal estimation rates for Plackett--Luce preference estimation, up to logarithmic factors. We assign \(500\) of these underlying offer sets to the training sample and the remaining \(100\) to the test sample. For each offer set, we generate two random permutations, determining the order in which its ten products are presented to the AI model. Both presentations of a given offer set are assigned to the same side of the training--test partition, so that product sets appearing in the test sample are not used for estimation. This randomization is intended to reduce the systematic association between specific products and specific positions in the presented list.

For each market segment and each randomized presentation, we instantiate the prompt using the corresponding persona instruction, product category, and ten-product offer set, and query the associated AI model for its five most preferred products from that set. Thus, each observation consists of a cardinality-\(10\) offer set together with an ordered partial ranking
\(
    (i_1,\ldots,i_5).
\)

The resulting design contains three product categories, \(600\) underlying
offer sets per category, two presentation orders per offer set, and nine
model--persona segments. Hence, the total number of observations is
\(
    3\times 600\times 2\times 9
    =
    32{,}400.
\)
Of these,
\(
    27{,}000
\)
belong to the training sample, while
\(
    5{,}400
\)
form the test sample. Equivalently, each category--model--persona cell contains \(1{,}000\) training rankings and \(200\) test rankings.

\subsection{Estimating the Plackett--Luce Model}
\label{sec:pl-estimation}
In this section, we present our estimation procedure for the stochastic
AI-ranking model and then report the resulting estimates and diagnostic
statistics. We estimate the Plackett--Luce parameters using a ridge-regularized likelihood criterion based on the standard PL likelihood and its partial-ranking variants; see, e.g.,
\citet{plackett1975analysis} and \citet{hajek2014minimax}.

\paragraph{Training data and likelihood function.}
For presentation purposes, let us consider a category with product universe \(N=\{1,\ldots,n\}\) and a market segment \(t\). Recall that segment \(t \in \cal T\) corresponds to a particular persona--AI-model pair.
We assume that the AI model's recommendations for this segment are generated according to a Plackett--Luce model with product-specific weights \(w_{it}>0\) for each product \(i\in N\).
Suppose that, for segment \(t\), we observe a set of independent training samples \(
    \mathcal D_t
    =
    \left\{
        \left(S_d,(i_{d,1},\ldots,i_{d,K})\right)
    \right\}_{d=1}^{D_t},
\)
where \(S_d\subseteq N\) is the offer set presented in observation \(d\), and \((i_{d,1},\ldots,i_{d,K})\) is the ordered top-\(K\) recommendation returned by the AI model.
Therefore, the likelihood of this observed set of training samples, given the Plackett--Luce weights \(w_{\cdot t}=(w_{1t},\ldots,w_{nt})\), is defined as the probability of jointly observing the ordered top-\(K\) rankings \((i_{d,1},\ldots,i_{d,K})\) for each offer set \(S_d\), \(d=1,\ldots,D_t\), assuming that AI recommendations are generated according to the Plackett--Luce model with weights \(w_{\cdot t}\). We denote this likelihood function by
\begin{equation*}
    \mathcal L_t(w_{\cdot t}\mid\mathcal D_t)
    =
    \Pr_{w_{\cdot t}}\!\left(
        (i_{d,1},\ldots,i_{d,K})_{d=1}^{D_t}
        \,\middle|\,
        (S_d)_{d=1}^{D_t}
    \right).
\end{equation*}
Using Lemma~\ref{lem:pl-ranking-probabilities}, we can express the likelihood of the observed partial rankings as
\begin{equation}
    \mathcal L_t(w_{\cdot t}\mid\mathcal D_t)
    =
    \prod_{d=1}^{D_t}
    \prod_{q=1}^{K}
    \frac{w_{i_{d,q}t}}{ \sum_{j \in S_d\setminus\{i_{d,1},\ldots,i_{d,q-1}\} } w_{jt}   }.
    \label{eq:pl-partial-likelihood}
\end{equation}
Substituting \(w_{it}=\exp(u_{it})\) and taking logarithms yields
\begin{equation}
    \ell_t(u_{\cdot t}\mid\mathcal D_t)
    =
    \log \mathcal L_t((\exp(u_{it}))_{i=1}^n\mid\mathcal D_t)
    =
    \sum_{d=1}^{D_t}
    \sum_{q=1}^{K}
    \left[
        u_{i_{d,q}t}
        -
        \log\!\left(
            \sum_{j\in
            S_d\setminus\{i_{d,1},\ldots,i_{d,q-1}\}}
            \exp(u_{jt})
        \right)
    \right].
    \label{eq:pl-log-likelihood}
\end{equation}

\paragraph{Concavity of the log-likelihood function.}
The utility parameterization above is particularly convenient from a computational perspective, since the resulting estimation problem is concave. Namely, the log-likelihood function above is concave in the product utilities \(u_{\cdot t}=(u_{1t},\ldots,u_{nt})\).

Indeed, for each observation \(d\) and rank \(q\), the term \(u_{i_{d,q}t}\) is linear in \(u_{\cdot t}\). Additionally, the function
$
    f: x\mapsto\log(\sum_j \exp x_j )
$
is convex. To see this, for any \(x,y \in\mathbb R^n\) and $t\in [0,1]$, we have
\begin{align*}
    f((1-t)x+ty)
    = \log\left(\sum_{i=1}^n \exp\left((1-t)x_i + ty_i\right)\right) &=
    \log\left(\sum_{i=1}^n \exp\left(x_i\right)^{1-t}\exp(y_i)^t\right)\\
    &\le
    \log\left(\left(\sum_{i=1}^n \exp\left(x_i\right)\right)^{1-t}\left(\sum_{i=1}^n \exp\left(y_i\right)\right)^t\right)\\
    &= (1-t)f(x)+tf(y),
\end{align*}
where the inequality follows from H\"older's inequality.
Hence the second term in each summand of
\eqref{eq:pl-log-likelihood} is convex. Each summand is therefore
concave, and summation preserves concavity.

\paragraph{Regularized log-likelihood.} Concavity makes the estimation problem computationally tractable, but it does
not by itself guarantee that the unregularized log-likelihood admits a finite
maximizer. In particular, the observed rankings may exhibit separation. For
example, suppose that, within a given market segment, a product is ranked first
every time it appears in a training offer set. Holding the remaining utilities
fixed, increasing the utility of this product increases the likelihood of each
such observation, and the likelihood may approach its supremum only as the
corresponding utility diverges. Thus, directly maximizing the unregularized
likelihood can yield an ill-posed estimation problem even though the
log-likelihood is concave.

To ensure a finite and well-defined estimate, we therefore use
ridge-regularized maximum likelihood. For each market segment \(t\), we
estimate the utility vector \(u_{\cdot t}\) by maximizing the log-likelihood minus a quadratic
penalty. Namely,
\begin{equation}
    \widehat u_{\cdot t}
    =
    \argmax_{u_{\cdot t}\in\mathbb R^n}
    \left\{
        \ell_t(u_{\cdot t}\mid\mathcal D_t)
        -
        \frac{\theta}{2}\lVert u_{\cdot t}\rVert_2^2
    \right\},
    \label{eq:pl-ridge-estimator}
\end{equation}
where \(\theta\) is a regularization parameter. Since the log-likelihood $\ell_t(\cdot \mid \mathcal D_t)$ is concave and the quadratic penalty \(\frac \theta 2\lVert\, \cdot\,\rVert_2^2\) is strictly convex, the resulting penalized objective is strictly concave and admits a unique finite maximizer.

Before evaluating the held-out test sample, we examined \(\theta\in\{0,\allowbreak 0.0025,\allowbreak 0.01,\allowbreak 0.04,\allowbreak 0.16,\allowbreak 0.64\}\) using only the training data, where \(\theta=0\) is the unregularized estimator. We chose \(\theta=0.04\) since it substantially reduced the most extreme differences between fitted utilities while increasing the average training negative log-likelihood (NLL) by less than \(0.01\) per ranking; larger values produced noticeably larger increases in training NLL. 
We solve each segment-specific optimization problem in Python using SciPy's L-BFGS-B routine with the analytic gradient, initializing all utilities at zero.

\subsection{Out-of-Sample Validation}
\label{sec:ranking-validation}

For this validation, we use the \(5{,}400\) held-out rankings from the three product categories. Table~\ref{tab:ai-ranking-validation} reports several
measures of predictive performance for the pooled sample, and separately by product category and AI model.
Our validation metrics are defined as follows.
\begin{itemize}[
    leftmargin=1.2em,
    itemsep=0.2em,
    topsep=0.3em,
    parsep=0pt,
    before={\setlength{\emergencystretch}{2em}}
]
    \item \textit{Held-out NLL.} This is the average negative log-likelihood that
    the fitted PL model assigns to the observed ordered top-five rankings in
    the held-out sample, using the likelihood function in
    \eqref{eq:pl-partial-likelihood}. Lower NLL means that the model assigns
    greater probability to the observed recommendations. Because its absolute scale is difficult to interpret, we compare held-out NLL with the training and uniform reference values below.

    \item \textit{Training NLL.} This is the same average negative
    log-likelihood computed on the observations used for estimation. It
    provides a direct reference for out-of-sample performance: a held-out NLL close
    to the training NLL indicates little deterioration on previously unseen
    offer sets, whereas a large increase would indicate degradation, implying that the model does not generalize well.

    \item \textit{Top-\(K\) set agreement.} For \(K\in\{1,3,5\}\), this is the
    fraction of held-out observations for which the unordered set of the AI's
    first \(K\) recommendations coincides exactly with the \(K\) offered
    products having the highest fitted PL utilities. Thus, order within the
    first \(K\) positions is ignored.

    \item \textit{Jaccard similarity.} The Jaccard similarity of two sets $A$ and $B$ is defined as the ratio ${|A \cap B|/|A \cup B|}$, and provides a measure of similarity of the two sets between $0$ and $1$. In Table~\ref{tab:ai-ranking-validation}, we report the average Jaccard similarity between the set of top five products recommended by the AI and the set of top five products predicted by the estimated PL model.
\end{itemize}

For reference, a PL model with equal worth for all ten offered products assigns
probability
\(
    1/(10\cdot 9\cdot 8\cdot 7\cdot 6)
\)
to any ordered top-five ranking. Its NLL is therefore
\(
    \log(10\cdot 9\cdot 8\cdot 7\cdot 6)=10.316921.
\)

\begin{table}[t]
\centering
\caption{Training fit and out-of-sample predictive performance of the estimated PL models.}
\label{tab:ai-ranking-validation}
\footnotesize
\begin{tabular*}{\linewidth}{@{\extracolsep{\fill}}lrrrrrrr@{}}
\toprule
Sample & Test rankings & Train NLL & Test NLL & Top-1 & Top-3 & Top-5 & Jaccard \\
\midrule
All categories      & 5,400 & 3.360 & 3.312 & 78.8\% & 68.4\% & 66.6\% & 0.879 \\
\addlinespace[2pt]
\multicolumn{8}{l}{\textit{By product category}} \\
Lightning chargers & 1,800 & 3.712 & 3.624 & 78.6\% & 64.5\% & 61.4\% & 0.859 \\
Robot vacuums       & 1,800 & 3.117 & 3.164 & 77.3\% & 70.6\% & 69.7\% & 0.891 \\
Pet leashes         & 1,800 & 3.249 & 3.148 & 80.6\% & 70.1\% & 68.8\% & 0.886 \\
\addlinespace[2pt]
\multicolumn{8}{l}{\textit{By AI model}} \\
ChatGPT              & 1,800 & 4.337 & 4.227 & 73.6\% & 59.3\% & 56.9\% & 0.838 \\
Claude           & 1,800 & 3.933 & 3.871 & 74.0\% & 64.0\% & 61.3\% & 0.860 \\
Gemini              & 1,800 & 1.809 & 1.837 & 88.9\% & 81.9\% & 81.7\% & 0.938 \\
\bottomrule
\end{tabular*}
\end{table}

\paragraph{Analysis.}
The pooled training and held-out NLLs are \(3.359555\) and \(3.311885\), respectively, indicating essentially no deterioration in predictive performance on offer sets that were not used for estimation. The held-out NLL is also far below the uniform benchmark of \(10.32\), as are all category- and model-level held-out NLLs.

The agreement metrics provide a complementary view. In the pooled sample, the
fitted rank center identifies the product ranked first by the AI in \(78.8\%\)
of held-out observations, reproduces the unordered first-three set in
\(68.4\%\), and exactly reproduces the unordered Top-5 set in \(66.6\%\). The
mean Top-5 Jaccard similarity is \(0.879\). Although the AI recommendations
exhibit stochastic variation, these results indicate that much of that
variation is organized around a stable preference ordering captured by the
estimated utilities. Performance varies across product categories and more
substantially across AI models. In this experiment, Gemini exhibits lower NLL and higher agreement than ChatGPT and Claude.

As noted at the beginning of this section, ordering products by their fitted PL utilities gives a natural deterministic approximation to the AI’s rankings. In Section~\ref{sec:numerics}, we optimize assortments under this approximation and evaluate their revenue under stochastic PL mediation.

\section{Experimental Effects of AI Mediation}
\label{sec:numerics}
We now embed the kernels estimated in Section~\ref{sec:ai-ranking-empirics} within a controlled semi-synthetic assortment experiment. Here, the AI-ranking component is estimated from observed LLM rankings, while the customer's MNL preferences are generated synthetically. As described below, customer utility combines a component aligned with the fitted AI-ranking preferences, capturing the part of customer preferences reflected in the AI recommendations, with an independent idiosyncratic component capturing preference variation not reflected in those recommendations. We use this experiment to study the following questions:
\begin{enumerate}
    \item[(i)] {\em How does mediation change the retailer's optimal assortment?}
    \item[(ii)] {\em How costly it is to ignore the mediation layer in terms of revenue?}
    \item[(iii)] {\em How much revenue is lost under PL mediation when the assortment is optimized using the deterministic ranking induced by the fitted AI utilities?}
\end{enumerate}

\subsection{Experimental Setup}
\label{sec:economic-experiment-design}
\label{sec:numerics-data}

\paragraph{Market segments.}
As previously mentioned, the market segments we consider represent unique combinations of customer persona and AI models. As such, it is useful to indicate these two components
in each market segment. Let \({\mathcal P} = \{\text{budget},\allowbreak\text{value},\allowbreak\text{premium}\}\) and \(\mathcal Q =  \{\text{ChatGPT},\allowbreak\text{Claude},\allowbreak\text{Gemini}\}\) denote the sets of personas and AI models, respectively. Throughout this section, rather than using the generic segment index \(t\), we index the nine
segments by \((p,q)\in\mathcal P\times\mathcal Q\), where \(p\) denotes the
persona and \(q\) the AI model. The segments receive equal probability,
\(\lambda_{p,q}=1/9\).

For segment \((p,q)\), the fitted PL utility \(u_{i,(p,q)}\), and hence the PL
weight \(w_{i,(p,q)}=\exp(u_{i,(p,q)})\), governs the AI's ranking model,
while \(v_{i,(p,q)}\) denotes the customer's MNL preference weight. We assume that
customer preferences depend on the persona but not on the AI model, i.e., $v_{i,(p,q)}
    =
    v_{i,(p,q')}$
for all $i\in N$, $p\in\mathcal P$ and $
    q,q'\in\mathcal Q$.
We therefore write \(v_{ip}=v_{i,(p,q)}\) for these common customer-side
weights. Thus, holding the persona fixed, variation across AI models changes the empirically estimated AI-ranking kernel but not the customer-choice primitives.

\paragraph{Synthetic construction of customer preferences.}
Customer preferences are not observed in the AI-ranking data, and we should therefore construct them synthetically. Our construction is designed to capture two features. First, since each AI is asked to recommend products for a given persona, it is natural for some component of that persona's customer preferences to be related to the preferences revealed by the fitted AI rankings. Second, the AI recommendations need not capture every component of customer preference. To reflect these features, we therefore construct customer utility from two components: a persona-level
component derived from the fitted AI-ranking utilities and an independent
Gaussian component representing idiosyncratic customer preference variation.
An alignment parameter \(\rho\in[0,1]\) controls the relative importance of
these two components.

Specifically, for each persona \(p\), we first standardize each model-specific fitted utility vector across products,
\[
    \widetilde u_{i,(p,q)}
    =
    \frac{u_{i,(p,q)}-\mu_{pq}}{\sigma_{pq}},
\]
where \(\mu_{pq}\) and \(\sigma_{pq}\) are the cross-product mean and standard deviation of \(u_{\cdot,(p,q)}\). This standardization puts all models’ utilities on a common scale, so that differences in their original scales do not give one model more weight than another when we average their preferences.
We then average the three standardized
vectors and standardize once more, with \(b_{ip}
    =
    \frac{1}{3}
    \sum_{q\in\mathcal Q}\widetilde u_{i,(p,q)},
    \qquad
\) and
\(    a_{ip}
    =
    \frac{b_{ip}-\bar b_p}{s_p},
\)
where \(\bar b_p\) and \(s_p\) are the cross-product mean and standard deviation of \(b_{\cdot p}\). The purpose of this second standardization becomes clear in the next step, where we generate customer utility as a weighted combination of these averaged preferences and independent standard Gaussian noise. Putting both components on the same variance scale allows the alignment parameter \(\rho\) to control their relative contributions.
The resulting \(a_{ip}\) is a persona-level summary of the fitted
AI preferences for product \(i\), obtained by averaging across
the three models.

For each product category, we generate multiple realizations of the random customer-preference component. Let \(\ell\) index these realizations, which we refer to as replications.
For replication \(\ell\), we generate customer utility according to
\begin{equation*}
    \beta_{ip}^{(\ell)}(\rho)
    =
    \rho a_{ip}
    +
    \sqrt{1-\rho^2}\,z_{ip}^{(\ell)},
\end{equation*}
where the \(z_{ip}^{(\ell)}\) are independent standard Gaussian random variables,
and the alignment parameter \(\rho\) varies in the set
\(
    \{0,.25,.50,.75,.90,.95,.99,1\}.
\)
The same Gaussian perturbation is reused across alignment values within a
category and replication, so comparisons across \(\rho\) are paired. {We choose the coefficient \(\sqrt{1-\rho^2}\) so that the squared weights of the two standardized components sum to one. This normalization keeps the overall variance scale of customer utility comparable as alignment varies.} For every customer persona $p \in {\cal P}$, we set the MNL preference weight of persona $p$ for product $i \in N$ in the $\ell^{\text{th}}$ replication as:
\[
    v_{ip}^{(\ell)}(\rho)
    =
    c\exp\!\left(\beta_{ip}^{(\ell)}(\rho)\right),
\]
with the outside-option weight normalized to one. Here, the common scale \(    c=0.248,\) is chosen to target an average customer product weight of approximately \(0.40\)
relative to the outside option. Thus, \(\rho=0\) makes customer utility
independent of the persona-level AI component, whereas \(\rho=1\) removes the idiosyncratic component and makes customer utility fully aligned with that component. Intermediate values interpolate between these two cases.

For every \(\rho<1\), we generate \(30\) replications per category. At
\(\rho=1\), the Gaussian component vanishes, meaning that all replications coincide and
only one endpoint instance is required per category. This design yields \(30 \times 7 + 1 = 211\) instances per category and \(633\) instances in total.

\paragraph{Policies.}
We consider three assortment-selection policies. Whenever a policy admits
multiple optimal assortments, we select one with smallest cardinality, and then the lexicographically smallest among those assortments.

\begin{itemize}
    \item \emph{Optimal PL policy.} The policy returns \(S_{\mathrm{PL}}^\star\), an optimal solution to \ref{prob:pl-ai-ao} under the fitted PL model.
    
    \item \emph{Deterministic plug-in policy.} The policy defines an instance of \ref{prob:d-ai-ao} with the same customer weights \(v_{ip}\) and the deterministic ranking induced by the fitted utilities \(u_{i,(p,q)}\). It returns \(S_{\mathrm D}^\star\), an optimal solution to this deterministic instance.
    
    \item \emph{Direct-MNL benchmark.} This policy ignores the AI mediation layer and returns \(S_{\mathrm{MNL}}^\star\), an optimal solution to \[
    \max_{S\in\mathcal F}
    \Rev^{\mathrm{MNL}}(S)
    =
    \sum_{(p,q)\in\mathcal P\times\mathcal Q}
    \lambda_{p,q}R_{(p,q)}(S),
    \] 
    where we recall that $\lambda_{p,q}$ is the arrival probability of segment $(p,q) \in {\cal P}\times {\cal Q}$, which is set to $1/9$ for all segments.
    Thus, the retailer optimizes as though its assortment were itself the customer-facing choice set. Since customers in each segment choose directly from the offered assortment according to their MNL model, this benchmark is equivalent to assortment optimization under a mixture of MNL models.
\end{itemize}
In this experiment, the assortments returned by all three policies above are computed through exhaustive enumeration of all feasible assortments.

\paragraph{Model parameters.} For each product \(i\), we set the retailer's revenue \(r_i\) equal to the listed product price in the underlying product metadata.
Recalling that \(B\) represents the retailer's assortment-size limit and \(K\) denotes the AI's display capacity, we solve every instance under
\(
    (B,K)\in
    \{(8,3),\allowbreak(8,5),\allowbreak(10,3),\allowbreak(10,5),\allowbreak(10,7),\allowbreak(12,5),\allowbreak(12,7)\}
\). Thus, the same \(633\) semi-synthetic instances are evaluated
in each of seven capacity environments, for \(4{,}431\)
configuration--instance observations.
For any candidate assortment \(S\), we measure performance by its relative revenue loss with respect to the exact PL-optimal assortment \(S_{\mathrm{PL}}^\star\), assuming a PL-mediation ground truth, defined as
\begin{equation}
    \operatorname{Loss}(S)
    =
    1-
    \frac{\Rev^{\mathrm{PL}}(S)}
         {\Rev^{\mathrm{PL}}(S_{\mathrm{PL}}^\star)}.
    \label{eq:policy-revenue-loss}
\end{equation}
Since $S^*_{\mathrm{PL}}$ is a maximizer of the expected revenue function $\Rev^{\mathrm{PL}}$, the above loss is in the interval $[0,1]$, and represents the fraction of revenue lost from an optimal solution assuming a PL-mediation ground truth.

\subsection{Experimental Effects of AI Mediation}
\label{sec:mediation-effects}

In this section, we study how the presence of the AI mediation layer affects the retailer's optimal assortment and the revenue consequences of ignoring that layer. Concurrently, we study how a deterministic plug-in policy performs relative to the optimal PL assortment.
We first pool across product categories, alignment values, and replications within each capacity environment. Each reported capacity-specific statistic therefore summarizes the same \(633\) underlying instances. We then condition on the alignment parameter \(\rho\) to study how the revenue effects of mediation change as customer preferences become more closely aligned with the AI-ranking preferences.

\subsubsection{Pooled Analysis}

\paragraph{Mediation regimes.}
\label{sec:numerics-regimes}
The first question we address is whether the introduction of AI mediation changes the customer-facing assortment compared to the direct-MNL benchmark. Indeed, our modeling assumptions do not preclude the possibility that the PL-optimal assortment coincides with the direct-MNL assortment. {In particular, if the PL-optimal assortment  $S^*_{\mathrm{PL}}$ is small enough to fit within the display capacity, then every offered product is displayed and the PL-mediated and direct-MNL revenues coincide at that assortment.}

This motivates a first distinction between two cases:
If \(|S_{\mathrm{PL}}^\star|>K\), the AI displays only a subset of the
offered products; if \(|S_{\mathrm{PL}}^\star|\le K\), every offered
product is displayed. However, the absence of filtering does not mean
that mediation has no effect: the retailer may choose a smaller
assortment precisely to avoid AI filtering.
Within the case \(|S_{\mathrm{PL}}^\star|\le K\), we therefore compare
the PL-optimal assortment with the direct-MNL benchmark. If
\(|S_{\mathrm{MNL}}^\star|>K\), the retailer contracts its assortment
relative to the direct-MNL benchmark so that every offered product is
displayed; we call this endogenous filtering avoidance. If
\(|S_{\mathrm{MNL}}^\star|\le K\), the retailer would choose an
assortment that fits within the display capacity even without
mediation; we call this structurally inactive mediation.
We thus classify each instance into three empirical regimes:
\begin{itemize}
    \item[(A)] Active filtering:
    \(|S_{\mathrm{PL}}^\star|>K\).

    \item[(B)] Endogenous filtering avoidance:
    \(|S_{\mathrm{PL}}^\star|\le K\) and
    \(|S_{\mathrm{MNL}}^\star|>K\).

    \item[(C)] Structurally inactive mediation:
    \(|S_{\mathrm{PL}}^\star|\le K\) and
    \(|S_{\mathrm{MNL}}^\star|\le K\).
\end{itemize}

\begin{figure}[!ht]
\centering
\includegraphics[width=0.60\linewidth]{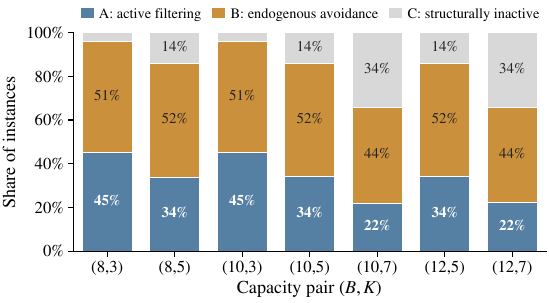}
\caption{Mediation regimes across the seven capacity environments.}
\label{fig:mediation-regimes}
\end{figure}
Figure~\ref{fig:mediation-regimes} shows that mediation is active throughout much of the capacity grid, either through realized filtering or through assortment contraction that avoids
filtering. At \(K=3\), regimes A and B
account for \(45.3\%\) and
\(50.7\%\) of instances, respectively, leaving only \(3.9\%\) in regime C. At \(K=5\), the corresponding shares are \(34.1\%\), \(51.7\%\), and \(14.2\%\). Even at \(K=7\), \(22.2\%\) of instances exhibit active filtering and \(43.8\%\) endogenous avoidance. Thus, a substantial part of the mediated share reflects endogenous avoidance rather than realized filtering. The resulting mediated share, A+B, declines from \(96.1\%\) at \(K=3\) to \(85.8\%\) at \(K=5\) and \(66.0\%\) at \(K=7\). Thus, within the tested grid, this coarse regime classification varies significantly with the display capacity \(K\) and is essentially invariant to the retailer assortment cardinality bound \(B\).

\paragraph{Strategic assortment contraction.}
\label{sec:numerics-contraction}
The regime decomposition points to a broader strategic response: a retailer that accounts for AI mediation often chooses a smaller assortment than one treating all offered products as directly customer-facing. Figure~\ref{fig:catalog-contraction} compares assortment size under the three policies.
In particular, Panel A reports the mean assortment size under each policy, while Panel B presents a box plot of \(|S_{\mathrm{MNL}}^\star|-|S_{\mathrm{PL}}^\star|\) across the 633 instances in each capacity setting. The outer lines indicate the minimum and maximum values. The lower and upper edges of each box represent the first and third quartiles, and the horizontal line inside the box represents the median.

Across capacity environments, the direct-MNL assortment exceeds the PL-optimal assortment by \(1.48\) to \(4.15\) products on average. The difference is largest when display capacity \(K\) is smaller. For example, at \((B,K)=(10,3)\), direct-MNL selects \(8.35\) products on average while the PL oracle selects only \(4.20\); at \((10,7)\), the corresponding averages are \(8.35\) and \(6.84\). The direct assortment is strictly larger than the PL assortment in \(96.0\%\) of the \(K=3\) instances, compared with \(57.2\%\) at \((10,7)\).

\begin{figure}[!ht]
\centering
\includegraphics[width=1\linewidth]
{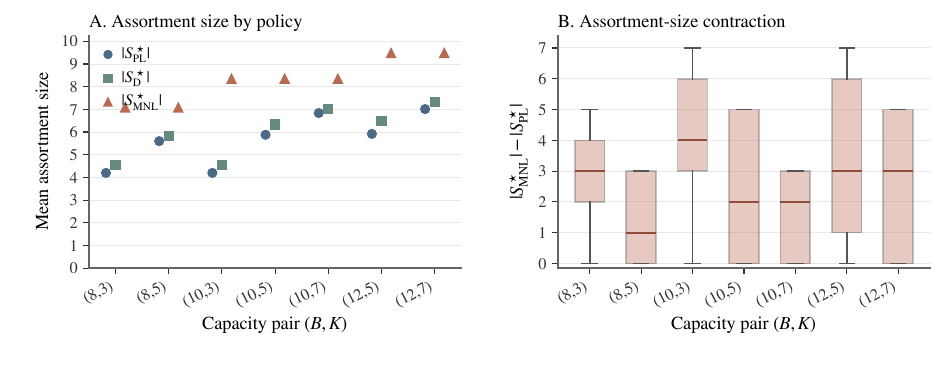}
\caption{Assortment size and assortment contraction across the seven capacity settings.}
\label{fig:catalog-contraction}
\end{figure}

This contraction has a direct interpretation in regime B. Whenever
\(|S|\le K\), every offered product is displayed and therefore \(
    \Rev^{\mathrm{PL}}(S)=\Rev^{\mathrm{MNL}}(S).
\)
Because \(S_{\mathrm{PL}}^\star\) itself lies in this all-display region in regime B, the direct-MNL problem with the additional constraint
\(|S|\le K\) attains the PL-optimal value. Regime B therefore isolates the
retailer's response to the display constraint of AI mediation: the mediation-aware retailer is
willing to forgo additional assortment capacity in order to remove the AI's discretion over which products reach the customer.

\paragraph{Revenue cost of ignoring mediation.}
\label{sec:numerics-policy-loss}
Evidently, these assortment changes have substantial revenue consequences.
Table~\ref{tab:bk-robustness} reports revenue loss separately in each
capacity environment. Mean direct-MNL revenue loss ranges from \(2.52\%\) to
\(11.37\%\), its \(95\)th percentile from \(10.36\%\) to \(30.66\%\), and its
maximum from \(23.83\%\) to \(57.08\%\). The largest losses occur when display
capacity is tight: mean direct-MNL revenue loss is \(9.80\%\) at \((8,3)\) and
\(11.37\%\) at \((10,3)\). In contrast, mean deterministic revenue loss ranges from
only \(0.08\%\) to \(0.53\%\) across the full grid, with \(95\)th percentiles
between \(0.59\%\) and \(2.70\%\).

\begin{table}[!b]
\centering
\caption{Policy performance across the seven capacity environments.}
\label{tab:bk-robustness}
\footnotesize
\begin{tabular}{rrrrrrrrrr}
\toprule
$B$ & $K$ & Mediated & Mean $|S^\star_{\mathrm{PL}}|$ & \multicolumn{3}{c}{Direct-MNL revenue loss} & \multicolumn{3}{c}{Deterministic revenue loss} \\
\cmidrule(lr){5-7} \cmidrule(lr){8-10}
 & & share & & Mean & P95 & Max & Mean & P95 & Max \\
\midrule
8 & 3 & 96.05\% & 4.20 & 9.80\% & 28.04\% & 57.08\% & 0.52\% & 2.70\% & 9.02\% \\
8 & 5 & 85.78\% & 5.60 & 4.04\% & 13.85\% & 37.96\% & 0.18\% & 0.99\% & 4.83\% \\
10 & 3 & 96.05\% & 4.20 & 11.37\% & 30.66\% & 52.11\% & 0.53\% & 2.70\% & 9.02\% \\
10 & 5 & 85.78\% & 5.88 & 5.50\% & 18.45\% & 37.90\% & 0.26\% & 1.45\% & 4.83\% \\
10 & 7 & 66.03\% & 6.84 & 2.52\% & 10.36\% & 23.83\% & 0.08\% & 0.59\% & 2.16\% \\
12 & 5 & 85.78\% & 5.92 & 6.53\% & 21.42\% & 37.90\% & 0.28\% & 1.59\% & 4.83\% \\
12 & 7 & 66.03\% & 7.01 & 3.56\% & 12.64\% & 23.83\% & 0.13\% & 1.00\% & 2.67\% \\
\bottomrule
\end{tabular}

\end{table}

Taken together, these experiments show that accounting for AI mediation often leads retailers to choose smaller assortments and can substantially improve revenue under PL mediation. The deterministic plug-in captures most of this benefit on average, with mean revenue loss of at most 0.53\% across the tested capacity environments.

\subsubsection{Role of Preference Alignment}
We next condition on \(\rho\) to examine how alignment between AI and customer preferences affects both the structure of the offered assortments and the expected revenue. For every \(\rho<1\), each
capacity point aggregates \(90\) instances, corresponding to \(30\) replications in each of the
three categories, while the \(\rho=1\) endpoint contains one category-specific
instance for each of the three categories.

\paragraph{Revenue loss and alignment.}
\label{sec:numerics-alignment}

In Figure~\ref{fig:alignment}, we report the mean revenue loss of the direct-MNL benchmark (Panel A) and the deterministic plug-in (Panel B) as a function of the alignment parameter \(\rho\). For clarity, this figure focuses on the three configurations with assortment-size limit \(B=10\); the corresponding curves for the remaining capacity configurations are reported in Appendix~\ref{app:alignment-full-grid}. The same qualitative patterns hold across the full capacity grid.

The most immediate observation is the substantial difference in scale between the two panels. At \(\rho=0\), for example, mean direct-MNL revenue loss is \(24.06\%\), \(11.83\%\), and \(5.75\%\) for \(K=3,5,7\), respectively, whereas the corresponding deterministic losses are only \(0.21\%\), \(0.03\%\), and essentially zero. The difference remains large at intermediate levels of alignment: at \(\rho=.75\), the direct-MNL losses are \(8.27\%\), \(3.89\%\), and \(1.80\%\), compared with only \(0.79\%\), \(0.23\%\), and \(0.07\%\) for the deterministic plug-in.
The gap narrows only as alignment becomes very high. Thus, over most of the alignment range, accounting only for the ordering induced by the fitted AI preferences recovers most of the revenue value of modeling mediation.

A second pattern concerns how costly it is to ignore mediation. Direct-MNL revenue loss falls sharply as customer preferences become more closely aligned with the AI-ranking preferences. Across the full capacity grid, mean direct-MNL revenue loss ranges from \(5.75\%\) to \(24.06\%\) at \(\rho=0\), compared with only \(0.59\%\) to \(5.74\%\) at \(\rho=1\); the endpoint loss is lower at \(\rho=1\) than at \(\rho=0\) in every configuration. This relationship is not perfectly monotone, as several configurations exhibit small reversals in monotonicity, but the overall decline is pronounced. When customer and AI preferences are poorly aligned, treating the retailer's assortment as though it were directly customer-facing can therefore be especially costly; as their preferences become aligned, the economic penalty of ignoring the mediation layer falls substantially.

The deterministic plug-in exhibits a different relationship with alignment. Its revenue loss remains small throughout, but is not monotone in \(\rho\) across all seven capacity configurations. In particular, greater customer--AI alignment can reduce direct-MNL loss while leaving deterministic loss unchanged or even increasing it. Hence, alignment explains when ignoring mediation is especially costly, but not when the deterministic approximation is most accurate. The latter remains a strong approximation across the alignment range, suggesting that the ordering induced by the fitted AI preferences captures substantial information about the mediation layer even when the
customer's preferences are only weakly aligned with those of the AI.
\begin{figure}[!t]
\centering
\includegraphics[width=\linewidth]
{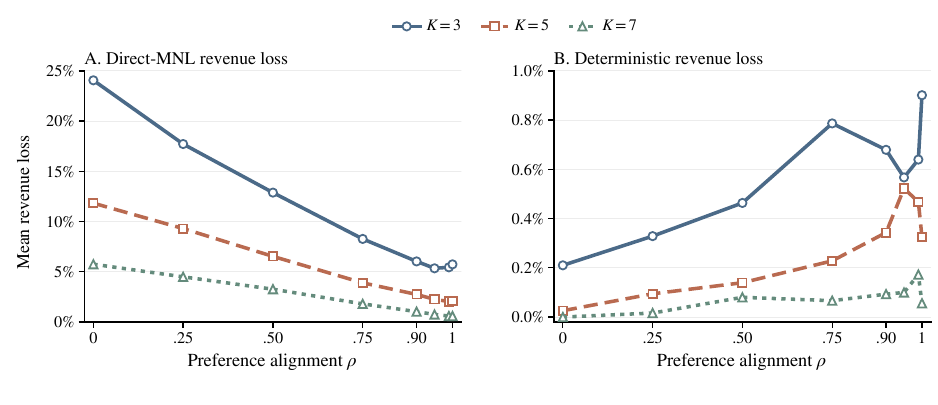}
\caption{Revenue loss by customer--AI preference alignment at \(B=10\).
Panel A reports mean direct-MNL revenue loss and Panel B reports mean
deterministic revenue loss, both for \(K\in\{3,5,7\}\).}
\label{fig:alignment}
\end{figure}

\paragraph{Assortment response to alignment.}
In Figure~\ref{fig:alignment-mechanisms}, we report the mean difference between the PL-optimal assortment size and the display capacity \(K\), and the share of instances in which the PL-optimal assortment exceeds the display capacity \(K\). Similar to Figure~\ref{fig:alignment}, the current figure focuses on the three \(B=10\) configurations, with the remaining capacity configurations reported in Appendix~\ref{app:alignment-filtering-mechanisms} for the sake of readability.

Figure~\ref{fig:alignment-mechanisms} shows that the reduction in direct-MNL revenue loss observed in Figure~\ref{fig:alignment} does not arise since mediation simply disappears at high alignment. In fact, it shows the opposite change in the mediation mechanism. At \(\rho=0\), active filtering occurs in at most \(8.9\%\) of instances across capacity environments, whereas endogenous avoidance accounts for \(62.2\%\) to \(84.4\%\).
As alignment rises, the PL-aware retailer becomes increasingly willing to offer more than \(K\) products. At \(\rho=.95\), active filtering accounts for \(44.4\%\) to \(66.7\%\) of instances, while endogenous avoidance remains sizeable at between \(22.2\%\) and \(33.3\%\).

The assortment-size statistic in Panel A draws a similar picture. At \(\rho=0\), mean \(\lvert S_{\mathrm{PL}}^\star\rvert-K\)
ranges from \(-0.74\) to \(0.13\) products across capacities; by
\(\rho=.95\), it ranges from \(0.37\) to \(1.96\). Thus, greater alignment
makes the retailer more willing to allow the AI greater discretion over which
offered products are displayed.

\begin{figure}[!t]
\centering
\includegraphics[width=\linewidth]
{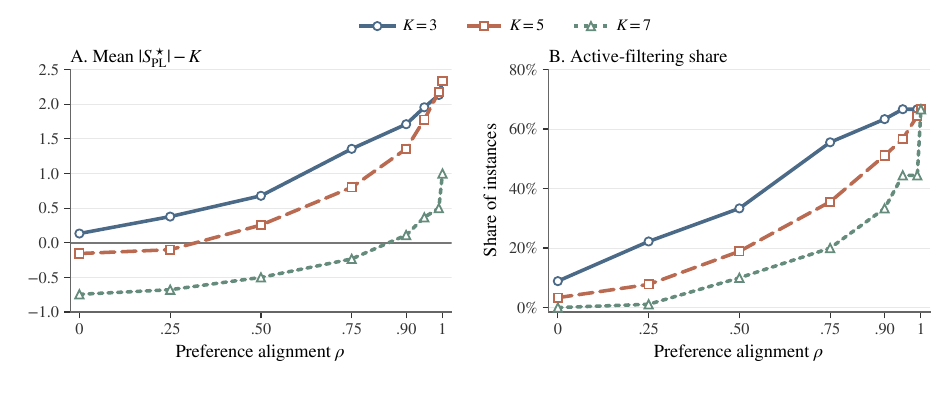}
\caption{Assortment response and mediation mechanism by preference alignment.
Panel A reports mean
\(\lvert S_{\mathrm{PL}}^\star\rvert-K\)
for each capacity environment. Panel B reports the corresponding share of
instances in the active-filtering regime.}
\label{fig:alignment-mechanisms}
\end{figure}

\section{Deterministic Ranking Mediation}
\label{sec:deterministic}

Section~\ref{sec:mediation-effects} shows that deterministic mediation serves as a practical surrogate for a PL-mediated ground truth.  This empirical performance makes
deterministic ranking mediation a natural operational surrogate. We now study the resulting assortment optimization problem, by characterizing its computational status under various primitive regimes, and developing a fully polynomial-time approximation scheme (FPTAS) for any fixed number of segments.

\subsection{Computational Complexity}
\label{sec:det-hardness}

In what follows, we characterize the computational complexity of \ref{prob:d-ai-ao} across the main regimes of the model parameters. Recall from Section~\ref{subsubsec:model-deterministic} that under deterministic mediation, we assume that each segment \(t\in {\cal T}\) is associated with a ranking \(\sigma_t \in \mathfrak S_N\),
and that the AI displays
\(
    C_t^{\mathrm D}(S)
    =
    \operatorname{TopK}_{\sigma_t}(S).
\)
The retailer's expected revenue is
\[
    \Rev^{\mathrm D}(S)
    =
    \sum_{t\in\mathcal T}
    \lambda_t
    R_t\!\left(C_t^{\mathrm D}(S)\right).
\]
Table~\ref{tab:det-computational-landscape} summarizes the resulting complexity landscape; we establish the complexity of each of its considered regimes, while the fixed-\(m\), arbitrary-\(K\) approximation result is developed in Section~\ref{sec:det-fptas}.

\begin{table}[htbp]
\centering
\small
\begin{tabular}{lccc}
\toprule
& \(m=1\) & Fixed \(m\ge 2\) & Arbitrary \(m\) \\
\midrule
Fixed \(K\)
& Polynomial
& Polynomial
& \(O(n^{1-\epsilon})\)-inapproximable \\
Arbitrary \(K\)
& Polynomial
& NP-hard; FPTAS
& \(O(n^{1-\epsilon})\)-inapproximable \\
\bottomrule
\end{tabular}
\caption{Computational complexity and approximability of deterministic
AI-mediated assortment optimization.}
\label{tab:det-computational-landscape}
\end{table}
We begin with the general setting, in which the number of segments \(m\) and the display capacity \(K\) may both scale with the input, and then examine the tractability that emerges when one or both of these parameters are fixed.

\paragraph{Arbitrary $\boldsymbol{m}$.} The next theorem shows that when the number of segments $m$ may scale with the input, deterministic AI mediation leads to a strongly-inapproximable assortment optimization problem even when the display capacity $K = 1$.

\begin{theorem}
\label{thm:det-hardness}
For any fixed $\eps>0$, \ref{prob:d-ai-ao} cannot be approximated in polynomial time within factor $O(n^{1-\eps})$, even when $K=1$ and $B=n$, unless $\mathrm{P} = \mathrm{NP}$.

\end{theorem}

The proof proceeds by a gap-preserving reduction from assortment optimization under ranking preferences, which is NP-hard to approximate within factor \(O(n^{1-\eps})\) for every fixed \(\eps>0\) \citep[Theorem~2.1]{aouad2018approximability}. We associate one customer segment with each preference list and set \(K=1\), using zero customer-side MNL weights for products outside the corresponding preference list. The resulting deterministic-mediation instance preserves the expected revenue of every assortment exactly. The complete reduction is provided in Appendix~\ref{app:det-hardness}.
We therefore turn to parameter regimes in which the number of segments or the display capacity is restricted.
\paragraph{Fixed \(\boldsymbol{m}\) and fixed \(\boldsymbol{K}\).}

Here, we consider the case in which both the number of segments $m$ and the display capacity $K$ are fixed. The key observation is that given an optimal assortment, products that are not displayed to any segment can be removed without affecting the expected revenue. Consequently, any optimal assortment can be transformed into another optimal assortment of size at most $mK$, by removing all such products.

\begin{lemma}
\label{lem:det-support-reduction}
There exists an optimal assortment \(S^*\in\mathcal F\) satisfying
\(
    |S^*|\le mK.
\)

\end{lemma}

\begin{proof}
Consider an optimal assortment \(S^*\in\mathcal F\), and define 
\(
    \widetilde S
    =
    C_1^{\mathrm D}(S^*)\cup \cdots\cup C_m^{\mathrm D}(S^*).
\)
Since each segment $t\in \cal T$ is displayed at most \(K\) products,
\(
    |\widetilde S|
    \le
    \sum_{t\in\mathcal T}
    \left|C_t^{\mathrm D}(S^*)\right|
    \le
    mK.
\)
Moreover, \(\widetilde S\subseteq S^*\), and hence
\(\widetilde S\in\mathcal F\).
To prove that $\widetilde S$ is optimal, it remains to observe that removing products in
\(S^*\setminus\widetilde S\) does not alter any segment's displayed set. Indeed, every product removed from \(S^*\) lies below the display cutoff for every segment, whereas all products in
\(C_t^{\mathrm D}(S^*)\) remain in \(\widetilde S\). Therefore, \(
    C_t^{\mathrm D}(\widetilde S)
    =
    C_t^{\mathrm D}(S^*)
\)
for every \(t\in\mathcal T\). Consequently,
$
    \Rev^{\mathrm D}(\widetilde S)
    =
    \Rev^{\mathrm D}(S^*).
$
\end{proof}

Lemma~\ref{lem:det-support-reduction} immediately yields a polynomial-time algorithm when \(m\) and \(K\) are fixed. It suffices to enumerate all feasible assortments consisting of at most \(mK\) products and evaluate their deterministic mediated revenues. The number of candidate assortments is \(O(n^{mK})\), which is polynomial in the input size for fixed \(m\) and \(K\).

\paragraph{Arbitrary \(\boldsymbol{K}\) and \(\boldsymbol{m=1}\).}
The single-segment problem remains tractable even when the display capacity $K$ is unbounded. By Lemma~\ref{lem:det-support-reduction}, there exists an optimal assortment containing at most \(K\) products. Such an assortment is displayed in its entirety to the unique segment, meaning that the current setting reduces to
$$
    \max_{\substack{S\subseteq N\colon\\
                    |S|\le \min\{B,K\}}}
    R_1(S).
$$
This is exactly single-segment cardinality-constrained MNL assortment optimization which is known to be solvable in polynomial time \citep[Theorem~2.2]{rusmevichientong2010dynamic}. Thus, allowing the display capacity $K$ to grow does not by itself create computational hardness.

\paragraph{Arbitrary \(\mathbf K\) and fixed \(\mathbf{m}\boldsymbol{\ge} \mathbf{2}\).}
As we show next, fixing the number of segments $m$ does not suffice to recover exact polynomial-time solvability when the display capacity is allowed to grow. Indeed, consider the special case of \(K=B=n\). Here, the AI displays every product in the retailer's assortment to every segment, so that \(C_t^{\mathrm D}(S)=S\) for every \(t\in\mathcal T\). The expected revenue reduces to
$$
    \Rev^{\mathrm D}(S)
    =
    \sum_{t\in\mathcal T}
    \lambda_t
    R_t(S).
$$
This is precisely the expected revenue function under a mixture of multinomial logit models. Therefore, any assortment optimization problem under the latter model can be mapped to an \ref{prob:d-ai-ao} instance, by retaining the same MNL customer preferences across all segments, setting $K=B=n$, and fixing an arbitrary ranking. Given that assortment optimization under a mixture of multinomial logit models is NP-hard even for $m=2$ segments \citep[Theorem~1]{rusmevichientong2014assortment}, \ref{prob:d-ai-ao} is NP-hard for fixed \(m\ge2\) when \(K\) is part of the input.

This final regime completes the exact-complexity landscape in Table~\ref{tab:det-computational-landscape}. Although exact optimization remains NP-hard when \(m\) is fixed and \(K\) is arbitrary, fixing \(m\) leaves only a constant number of segment-specific revenue components. We exploit this structure in the next section to obtain an FPTAS.

\subsection{An FPTAS for a Fixed Number of Segments}
\label{sec:det-fptas}

In this section, we show that fixing \(m\) permits a fully
polynomial-time approximation scheme.   

\begin{theorem}
\label{thm:det-fptas}
For every fixed \(m\) and every \(\epsilon\in(0,1)\), a feasible
assortment \(\widehat S\in\mathcal F\) satisfying
\[
\Rev^{\mathrm D}(\widehat S)
\ge
(1-\epsilon)\max_{S\in\mathcal F}\Rev^{\mathrm D}(S)
\]
can be computed in time polynomial in the binary encoding length
of the instance and \(1/\epsilon\).
\end{theorem}

\paragraph{Preliminaries.}
Among all optimal assortments for Problem~\ref{prob:d-ai-ao},
let \(S^*\) have minimum cardinality.
We say that mediation is \emph{inactive} for an assortment \(S\) if \(|S|\le K\), in which case every segment views the entire assortment, and \emph{active} if \(|S|>K\), in which case each segment views only its \(K\) highest-ranked products.
Since our algorithm does not know which regime contains
\(S^*\), we propose a separate procedure for each case and return the candidate assortment with the largest value of expected revenue. We first handle the easy case of inactive mediation.

For every assortment \(S\) with \(|S|\le\min\{B,K\}\), each segment views $S$ in its entirety, since \(C_t^{\mathrm D}(S)=S\) for all segments \(t\in\mathcal T\). Hence, \(\Rev^{\mathrm D}(S)=\sum_{t\in\mathcal T}\lambda_t R_t(S)\), and optimization over these assortments is precisely the cardinality-constrained assortment optimization problem under a mixture of multinomial logit models, with capacity \(\min\{B,K\}\).
For fixed \(m\), the FPTAS of \citet{desir2022capacitated} therefore
returns a feasible assortment \(\bar S\) satisfying
\[
\Rev^{\mathrm D}(\bar S)
\ge
(1-\epsilon)\cdot
\max_{\substack{S\subseteq N\colon\\ |S|\le\min\{B,K\}}}
\Rev^{\mathrm D}(S).
\]
The running time of the algorithm therein is
\begin{equation}\label{eq:time_inactive}
\left(
\frac{n\bigl(L+\log(1/\epsilon)\bigr)}{\epsilon}
\right)^{O(m)},
\end{equation}
where \(L\) denotes the binary encoding length of the instance.
If an optimal assortment \(S^*\) for
Problem~\ref{prob:d-ai-ao} satisfies \(|S^*|\le K\), it belongs
to this restricted feasible family, so
\(\Rev^{\mathrm D}(\bar S)
\ge(1-\epsilon)\Rev^{\mathrm D}(S^*)\).
In the remainder of this section, we focus on the active regime, when $|S^*|>K$.

It remains to construct a candidate with the same guarantee when \(|S^*|>K\), which we address next. In order to do so we proceed in four steps. In Section~\ref{sec:det-fptas-thresholds}, we introduce an alternative formulation for \ref{prob:d-ai-ao} by guessing the lowest ranked product in each segment's displayed set under the optimal assortment $S^*$. In Section~\ref{sec:det-fptas-framework}, we introduce an approximation framework due to \citet{mittal2013general} for optimizing a sum of a fixed number of ratios of affine functions. In Section~\ref{sec:det-fptas-exact}, we use a multidimensional knapsack reformulation to show that the conditions of the framework of Section~\ref{sec:det-fptas-framework} are satisfied. Finally, in Section~\ref{sec:det-fptas-completion}, we combine these ingredients to obtain an FPTAS for \ref{prob:d-ai-ao} when \(m\) is fixed.

\subsubsection{A threshold reformulation}
\label{sec:det-fptas-thresholds}
Suppose now that \(|S^*|>K\). Our first step is to guess,
through exhaustive enumeration, the lowest-ranked product in each segment's displayed set under \(S^*\). Fixing an arbitrary vector \(\boldtau=(\tau_t)_{t\in\mathcal T}\in N^m\), we interpret
each \(\tau_t\) as a candidate for the lowest-ranked product in \(C_t^{\mathrm D}(S^*)\). We call \(\tau_t\) the
\emph{display threshold} for segment \(t \in \cal T\). There are \(n^m\)
candidate vectors, so let's assume that $\boldtau^*$, the vector of display thresholds corresponding to the optimal assortment $S^*$, is known as we proceed.

For any product $i\in N$ and customer segment $t\in {\cal T}$, let \(\operatorname{rank}_t(i)\) denote the position of product \(i\) in the strict ranking \(\sigma_t\), with smaller positions indicating higher ranks. Let
\(
h_{it}
=
\mathds{1}\!\left\{
\operatorname{rank}_t(i)\le\operatorname{rank}_t(\tau_t^*)
\right\}
\), which indicates whether product \(i\) ranks at least as high as the optimal display threshold $\tau_t^*$. To realize these thresholds, we require every threshold product
to be selected and exactly \(K\) selected products to satisfy
\(h_{it}=1\) for each segment \(t \in \cal T\). Let  \(M=\{\tau_t^*:t\in\mathcal T\}\) be the
set of distinct threshold products and let
\(
A
=
\{
i\in N:\sum_{t\in\mathcal T}h_{it}>0
\}
\) be set of products that rank at least as high as some threshold.
The set \(M\) records the products whose inclusion
is required, counting a product only once even if it serves as the
threshold for several segments. Products outside
\(A\) lie below every proposed threshold.
If these thresholds are realized, such products are not displayed to any segment and can be omitted without affecting revenue. We therefore
restrict attention to assortments satisfying
\(M\subseteq S\subseteq A\).
Finally, we define \(a_{it}=r_i v_{it}h_{it}\) and
\(b_{it}=v_{it}h_{it}\), and we introduce the following optimization problem:

\begin{equation}
\begin{aligned}
&\max_{S\subseteq A}\quad&
    &\Rev^{\boldtau^*}(S)
    =
    \sum_{t\in\mathcal T}\lambda_t \cdot
    \frac{\sum_{i\in S} a_{it}}
         {1+\sum_{i\in S} b_{it}}
\\
&\text{s.t}\quad&
    &M\subseteq S
    && \qquad \text{(i)}
\\
&&
    &|S|\le B
    && \qquad \text{(ii)}
\\
&&
    &\sum_{i\in S}h_{it}=K
    \qquad \text{for all }t\in\mathcal T
    && \qquad \text{(iii)}
\end{aligned}
\tag{$\text{{\sf AIM}}_\text{\sf D}^{\boldtau^*}$}
\label{eq:det-fptas-conditional}
\end{equation}

For ease of notation, let \(\mathcal X_{\boldtau^*}\) denote the feasible family
of Problem~\eqref{eq:det-fptas-conditional}, consisting of the assortments
\(S\subseteq A\) satisfying constraints~(i)--(iii).
The next lemma links this conditional problem to
Problem~\ref{prob:d-ai-ao}. It shows that their objective functions coincide
over \(\mathcal X_{\boldtau^*}\), and that the optimal assortment \(S^*\)
is feasible for the conditional problem. Its proof is provided in
Appendix~\ref{app:det-fptas-thresholds}.

\begin{lemma}
\label{lem:det-fptas-thresholds}
For every \(S\in\mathcal X_{\boldtau^*}\), we have \(
\Rev^{\boldtau^*}(S)=\Rev^{\mathrm D}(S).
\)
Moreover, \(
S^*\in\mathcal X_{\boldtau^*}.
\)
\end{lemma}

Lemma~\ref{lem:det-fptas-thresholds} implies that any
\((1-\epsilon)\)-approximate solution to
Problem~\ref{eq:det-fptas-conditional} is also a
\((1-\epsilon)\)-approximate solution to Problem~\ref{prob:d-ai-ao}.
Indeed, let \(\widehat S\in\mathcal X_{\boldtau^*}\) be such a solution.
Then,
\begin{align}
\Rev^{\mathrm D}(\widehat S)
&=
\Rev^{\boldtau^*}(\widehat S)
\label{eq:det-fptas-threshold-transfer-output}\\
&\ge
(1-\epsilon)
\max_{S\in\mathcal X_{\boldtau^*}}
\Rev^{\boldtau^*}(S)
\label{eq:det-fptas-threshold-transfer-approx}\\
&\ge
(1-\epsilon)\Rev^{\boldtau^*}(S^*)
\label{eq:det-fptas-threshold-transfer-feasibility}\\
&=
(1-\epsilon)\Rev^{\mathrm D}(S^*).
\label{eq:det-fptas-threshold-transfer-optimum}
\end{align}
Equation~\eqref{eq:det-fptas-threshold-transfer-output} follows from
Lemma~\ref{lem:det-fptas-thresholds}.
Equation~\eqref{eq:det-fptas-threshold-transfer-approx} follows from the
approximation guarantee for \(\widehat S\).
Equation~\eqref{eq:det-fptas-threshold-transfer-feasibility} follows from
\(S^*\in\mathcal X_{\boldtau^*}\).
Finally, Equation~\eqref{eq:det-fptas-threshold-transfer-optimum} follows
again from Lemma~\ref{lem:det-fptas-thresholds}.

It therefore remains to construct a \((1-\epsilon)\)-approximate solution
to Problem~\ref{eq:det-fptas-conditional}, which we address in the
remainder of this section.

\subsubsection{A framework for sums of ratios}
\label{sec:det-fptas-framework}

We next introduce the approximation framework of
\citet{mittal2013general} for optimizing a sum of a fixed number
of ratios of affine functions. We proceed by introducing the general setup.

Let \(\mathcal X\subseteq\{0,1\}^d\) be a nonempty
feasible family. For each \(t\in\{1,\ldots,m\}\), let
\(\alpha_{t0},\beta_{t0}\in\mathbb R_{\ge0}\) and
\(\alpha_t,\beta_t\in\mathbb R_{\ge0}^d\).
We consider the optimization problem
\begin{equation}
\label{eq:det-fptas-sum-of-ratios}
\max_{x\in\mathcal X}\; G(x)
=
\sum_{t=1}^m
\frac{\alpha_{t0}+\alpha_t^\top x}
     {\beta_{t0}+\beta_t^\top x}.
\end{equation}
We assume that every numerator and denominator is strictly
positive on \(\mathcal X\), and let \(L\) denote the total
binary encoding length of the feasible-family description
and the objective coefficients.
Next, we define the {\em exact problem} associated with the family $\mathcal X$ as: given
\(c\in\mathbb Z_{\ge0}^d\) and \(C\in\mathbb Z_{\ge0}\),
determine whether there exists \(x\in\mathcal X\) satisfying
\(
\sum_{i=1}^d c_i x_i=C.
\)
Let \(T_{\mathrm{ex}}(L,U)\) be an upper bound on the running time of an algorithm that solves the exact problem, where \(U\) is an upper bound on the absolute values of the entries of \(c\) and \(C\).

The next theorem, due to \citet{mittal2013general}, shows that if the exact problem is solvable in pseudopolynomial time, then the sum-of-ratios problem~\eqref{eq:det-fptas-sum-of-ratios} admits a fully polynomial-time approximation scheme.

\begin{theorem}%
\label{thm:det-fptas-framework}
For every \(\epsilon\in(0,1)\), a feasible solution \(\widehat x\in\mathcal X\) satisfying
\(
G(\widehat x)
\ge
(1-\epsilon)\max_{x\in\mathcal X}G(x)
\)
can be computed in time\[
\left(\frac{(d+1)L}{\epsilon}\right)^{O(m)}
T_{\mathrm{ex}}\!\left(
L,\left(\frac{d+1}{\epsilon}\right)^{O(m)}
\right).
\]
\end{theorem}
The above theorem follows from Theorem~7.5 of \citet{mittal2013general}, which establishes the result for a single ratio. The subsequent discussions in Section~7 therein extend the result to the sum of a fixed number of ratios.
The displayed running-time bound is obtained by tracking
their construction.

\subsubsection{Verifying the exact-solvability condition}
\label{sec:det-fptas-exact}
We show that Theorem~\ref{thm:det-fptas-framework} can be applied to obtain an FPTAS for \ref{eq:det-fptas-conditional}. First, since the numerators and denominators in the framework of Theorem~\ref{thm:det-fptas-framework} must be strictly positive for all feasible solutions, we slightly shift the objective function by adding a strictly positive term to ensure this condition is met. Then, we show that the exact problem associated with \ref{eq:det-fptas-conditional} can be solved in pseudopolynomial time by a reduction to a multidimensional knapsack problem.

\paragraph{Objective shift.}
Recall the objective function of~\ref{eq:det-fptas-conditional}:
\[
\Rev^{\boldtau^*}(S)
=
\sum_{t\in\mathcal T}\lambda_t
\frac{\sum_{i\in S}a_{it}}
     {1+\sum_{i\in S}b_{it}}.
\]
The denominators are strictly positive for all feasible
assortments, but the numerator for a segment
\(t\in\mathcal T\) is zero if \(a_{it}=0\) for every
\(i\in S\). To address this issue, we add a strictly
positive term to each numerator.
Let
\(
\ell=\max_{i\in N}\Rev^{\mathrm D}(\{i\})
\)
denote the maximum expected revenue of a singleton
assortment. Since every singleton assortment is feasible
for the original problem, we have
\(\ell\le\Rev^{\mathrm D}(S^*)\).
If \(\ell=0\), every feasible assortment has zero revenue,
so we may assume that \(\ell>0\).
For every \(S\in\mathcal X_{\boldtau^*}\), we introduce
the shifted revenue function
\[
\widehat{\Rev}^{\boldtau^*}(S)
=
\sum_{t\in\mathcal T}\lambda_t
\frac{\sum_{i\in S}a_{it}
      +\ell\left(1+\sum_{i\in S}b_{it}\right)}
     {1+\sum_{i\in S}b_{it}}
=
\Rev^{\boldtau^*}(S)+\ell,
\]
where the second equality follows from
\(\sum_{t\in\mathcal T}\lambda_t=1\).
Each numerator is now at least \(\ell>0\).
We also introduce the shifted problem
\begin{equation}
\label{eq:det-fptas-shifted}
\tag{$\widehat{\mathrm{AIM}}_{\boldtau^*}$}\max_{S\in\mathcal X_{\boldtau^*}}
\widehat{\Rev}^{\boldtau^*}(S).
\end{equation}
Let \(\widehat S^*\) be an optimal solution
to~\ref{eq:det-fptas-shifted}.
The following lemma shows that an approximation guarantee
for the shifted objective yields one for the original
objective.

\begin{lemma}
\label{lem:det-fptas-shift}
For every \(\epsilon\in(0,1)\) and
\(\widehat S\in\mathcal X_{\boldtau^*}\), if
\(
\widehat{\Rev}^{\boldtau^*}(\widehat S)
\ge
(1-\epsilon/2)
\widehat{\Rev}^{\boldtau^*}(\widehat S^*),
\)
then
\(
\Rev^{\mathrm D}(\widehat S)
\ge
(1-\epsilon)\Rev^{\mathrm D}(S^*).
\)
\end{lemma}

\begin{proof}
By Lemma~\ref{lem:det-fptas-thresholds},
\(S^*\in\mathcal X_{\boldtau^*}\), and
\(\Rev^{\boldtau^*}(S)=\Rev^{\mathrm D}(S)\)
for every \(S\in\mathcal X_{\boldtau^*}\).
Consequently,
\begin{align*}
\Rev^{\mathrm D}(\widehat S)
&=
\widehat{\Rev}^{\boldtau^*}(\widehat S)-\ell\\
&\ge
(1-\epsilon/2)
\widehat{\Rev}^{\boldtau^*}(\widehat S^*)-\ell\\
&\ge
(1-\epsilon/2)
\widehat{\Rev}^{\boldtau^*}(S^*)-\ell\\
&=
(1-\epsilon/2)\Rev^{\mathrm D}(S^*)
-\frac{\epsilon}{2}\ell\\
&\ge
(1-\epsilon)\Rev^{\mathrm D}(S^*),
\end{align*}
where the last inequality uses
\(\ell\le\Rev^{\mathrm D}(S^*)\).
\end{proof}
Consequently, a \((1-\epsilon/2)\)-approximate solution
for the shifted problem yields a
\((1-\epsilon)\)-approximate solution for the original
problem.

\paragraph{Algorithm for the exact problem.}
In light of Lemma~\ref{lem:det-fptas-shift}, we focus on
approximating the shifted problem~\ref{eq:det-fptas-shifted}.
Identifying each assortment \(S\subseteq A\) with its
incidence vector \(x\in\{0,1\}^{A}\), where \(x_i=1\)
if and only if \(i\in S\) for every \(i\in A\), we
can rewrite~\ref{eq:det-fptas-shifted} as the following
integer program:
\begin{equation}
\label{eq:det-fptas-shifted-vector}
\begin{aligned}
\max\quad
& \sum_{t\in\mathcal T}\lambda_t
  \frac{\ell+\sum_{i\in A}(a_{it}+\ell b_{it})x_i}
       {1+\sum_{i\in A}b_{it}x_i}\\
\text{s.t.}\quad
& x_i=1,
&& i\in M,\\
& \sum_{i\in A}x_i\le B,\\
& \sum_{i\in A}h_{it}x_i=K,
&& t\in\mathcal T,\\
& x_i\in\{0,1\},
&& i\in A.
\end{aligned}
\end{equation}
Let \(\mathcal X\subseteq\{0,1\}^{A}\) denote the
feasible set of~\eqref{eq:det-fptas-shifted-vector}.
We now show that this problem falls within the framework
of Theorem~\ref{thm:det-fptas-framework}.
Each numerator is at least \(\ell>0\), and each
denominator is at least one. Moreover, both are affine
functions of \(x\) with nonnegative coefficients.
It remains to provide a pseudopolynomial-time algorithm
for the associated exact problem.

\noindent The exact problem associated with
\eqref{eq:det-fptas-shifted-vector} is to determine, for a given
vector \(c\in\mathbb Z_{\ge0}^{A}\) and target
\(C\in\mathbb Z_{\ge0}\), whether there exists
\(x\in\mathcal X\) satisfying \(c^\top x=C\). We show that this problem can be solved in pseudopolynomial time by a reduction to a multidimensional knapsack problem.
First notice that any feasible solution \(x\in\mathcal X\) must satisfy \(x_i=1\) for all \(i\in M\), and have at most $B$ entries equal to one. Letting $I = A\setminus M$, the exact problem is equivalent to: determine whether there exists a  vector \(x\in\{0,1\}^{I}\) satisfying
\[
\sum_{i\in I}x_i\le B-|M|,
\qquad
\sum_{i\in I}c_i x_i=C-\sum_{i\in M}c_i,
\]
and, for every \(t\in\mathcal T\),
\[
\sum_{i\in I}h_{it}x_i=K-\sum_{i\in M}h_{it}.
\]

If any right-hand side above is negative, the exact
problem is infeasible. Otherwise, we construct the
following \((m+2)\)-dimensional \(0\)--\(1\) knapsack
problem:
\begin{equation}
\begin{aligned}
\max_{x\in\{0,1\}^{I}}\quad
&\sum_{i\in I}(c_i + \sum_{t\in\mathcal T}h_{it})\cdot x_i\\
\text{s.t.}\quad
&\sum_{i\in I}x_i\le B-|M|,\\
&\sum_{i\in I}h_{it}x_i\le K - \sum_{i\in M}h_{it}
\qquad\text{for all }t\in\mathcal T,\\
&\sum_{i\in I}c_i x_i\le C - \sum_{i\in M}c_i.
\end{aligned}
\label{eq:det-fptas-knapsack}
\end{equation}

\noindent The next lemma links problem~\eqref{eq:det-fptas-knapsack} to the exact problem.
\begin{lemma}\label{lem:det-fptas-knapsack}
    The  exact problem is feasible if and only if \eqref{eq:det-fptas-knapsack} admits a feasible solution with profit at least \[C-\sum_{i\in M}c_i+ mK - \sum_{t\in\mathcal T}\sum_{i\in M} h_{it}.\]
\end{lemma}

\noindent The proof is given in Appendix~\ref{app:det-fptas-knapsack}.
The standard exact algorithm for multidimensional knapsack
\citep[Section~9.3.2]{kellerer2004knapsack}
solves the constructed instance and recovers a selected
subset in time
\(
O\!\left(
m(C+1)n^{O(m)}
\right).
\)

\subsubsection{Completing the proof}
\label{sec:det-fptas-completion}
\paragraph{Correctness.}
For each candidate threshold vector, we apply
Theorem~\ref{thm:det-fptas-framework} to the shifted
problem~\ref{eq:det-fptas-shifted} with accuracy
\(\epsilon/2\). For the threshold vector
\(\boldtau^*\), Lemma~\ref{lem:det-fptas-shift}
guarantees that the returned assortment has expected
revenue at least
\((1-\epsilon)\Rev^{\mathrm D}(S^*)\).
Evaluating all returned assortments, together with
\(\bar S\), under the original revenue function and
retaining the best therefore gives the desired
approximation guarantee.

\paragraph{Running time.}
There are at most \(n^m\) candidate threshold vectors.
For each vector, we solve the corresponding shifted
problem~\ref{eq:det-fptas-shifted}.
Let \(L\) denote the length of the original instance,
and let \(\widehat L=O(nmL)\) bound the lengths of the
shifted conditional instances.
The exact algorithm satisfies
\[
T_{\mathrm{ex}}(\widehat L,U)
=
O\!\left(m(U+1)\widehat L^{O(m)}\right).
\]
Therefore, substituting this bound and \(d=|A|\le n\)
into Theorem~\ref{thm:det-fptas-framework}, with
accuracy \(\epsilon/2\), gives a running time of
\(
m\left({n\widehat L}/{\epsilon}\right)^{O(m)}
\)
per threshold vector. Multiplying by the number of
threshold vectors gives
\[
mn^m
\left(\frac{n\widehat L}{\epsilon}\right)^{O(m)}
=
m\left(\frac{nmL}{\epsilon}\right)^{O(m)}.
\]
Including the inactive-mediation procedure and the
evaluation and comparison of the returned assortments,
the total running time is
\[
m\left(
\frac{nm\left(L+\log(1/\epsilon)\right)}{\epsilon}
\right)^{O(m)}.
\]
For fixed \(m\), this bound is polynomial in \(L\) and
\(1/\epsilon\), establishing the FPTAS.

\section{Conclusion}\label{sec:conclusion}

This paper studies assortment optimization when a retailer chooses the catalog available to an AI intermediary, which then determines the products shown to each customer. We estimate and validate Plackett--Luce ranking models using controlled queries of AI shopping assistants. Combining these estimates with synthetic customer preferences, we find that accounting for mediation leads to smaller assortments on average and can induce retailers to limit their catalogs so that every offered product is displayed. Across the tested capacity configurations, ignoring mediation produces mean revenue losses ranging from \(2.52\%\) to \(11.37\%\) relative to the PL optimum.
A deterministic approximation based on the ordering of fitted AI utilities recovers nearly all of the optimal PL-mediated revenue on average in our experiments, with mean losses of at most \(0.53\%\) across capacity configurations. This finding motivates our analysis of assortment optimization under deterministic mediation. We characterize its computational complexity and develop an FPTAS for any fixed number of customer segments and arbitrary display capacity. The FPTAS guarantee applies to deterministic mediation. The numerical comparisons under stochastic mediation use exactly optimized deterministic assortments.

We hope this work opens new directions for research on assortment optimization under AI mediation. As AI shopping assistants become more prevalent, understanding how they reshape product exposure, consumer choice, and retailers’ decisions will become increasingly important.

\section*{Acknowledgments}
The authors gratefully acknowledge Danny Segev for valuable discussions and thoughtful feedback on this paper.

{\small\singlespacing
\bibliographystyle{plainnat}
\bibliography{references}
}

\newpage
\appendix
\section{Additional Material for Section 3}
\label{app:estimation-computation}

\subsection{AI-Ranking Data Generation Details}
\label{app:pl-calibration}

{\emergencystretch=2em
\paragraph{Product examples.}
To illustrate the information supplied to the AI models, we showcase one representative product from each product category. Product numbers are category-specific identifiers used in the ranking task. The prompt includes the product title, listed price, brand/store, and the available category-specific attributes; fields with missing values are omitted.
\begin{itemize}
    \item \textbf{Lightning charger (product 1).} 
        \emph{Title:} {(MFi Apple Certified) Galvanox 1-PC Lightning Fast
        iPhone Wall Charger Plug for iPhone 11/12/13/14 Pro Max (20W PD
        Technology)}.
        \emph{Listed price:} \$21.49.
        \emph{Brand/store:} GALVANOX.
        \emph{Category-specific attributes:}
        \texttt{wattage=20} and
        \texttt{connector\_type=lightning}.
    \item \textbf{Robot vacuum (product 1).}
        \emph{Title:} {Robot Vacuum Cleaner,Sweeping Robot,Ultra Slim
        Quiet,1800pa Super-Strong Suction,Cleans Hard Floors to Medium-Pile
        Carpets,Integral Memory Multiple Cleaning Modes Vacuum Best for Pet Hairs}.
        \emph{Listed price:} \$73.59.
        \emph{Brand/store:} Green Hills.
        \emph{Category-specific attributes:}
        \texttt{suction\_power\_\allowbreak pa=1800},
        \texttt{run\_time\_\allowbreak min=90},
        \texttt{wifi=not\_\allowbreak detected},
        \texttt{mopping=no}, and
        \texttt{self\_\allowbreak emptying=no}.
    \item \textbf{Pet leash (product 1).} \emph{Title:} {Pet Attire\textsuperscript{\textregistered} Ribbon Nylon Leash, 5/8\char34{} x 6\char39{}}.
        \emph{Listed price:} \$9.87.
        \emph{Brand/store:} Coastal.
        \emph{Category-specific attributes:}
        \texttt{product\_type=leash},
        \texttt{species=pet}, and
        \texttt{material=nylon}.
\end{itemize}

\paragraph{Prompt template.}
For reproducibility, we report below the complete prompt template for the
value-seeking persona. Bracketed fields are populated separately for each
ranking task using the relevant product category, the ten product records in
the presented offer set, and their product numbers. The budget-sensitive and
premium-seeking prompts use the same template, with only the opening persona
instruction replaced by the corresponding persona description.

\smallskip
\noindent\emph{System message.}
\begin{quote}
\small\ttfamily\raggedright
You are assisting a value-seeking customer who cares most about the best
price-to-quality tradeoff. They will pay more for a meaningful step up in
quality, but do not want to overpay for marginal improvements or for brand
name alone.

\medskip
You are a personal shopping assistant evaluating [CATEGORY LABEL] products.

\medskip
TASK: You will receive a list of products available in a specific offer set.
Select the best 5 products from ONLY those offered items.

\medskip
RULES:

- Return valid JSON only --- a single object with a top\_5 array.

- Return exactly five entries, ordered by preference, with ranks exactly 1
through 5.

- Each element must contain only rank and product\_number.

- product\_number MUST be one of the catalog numbers listed in the offer set.
Do not return any product\_number that was not offered.

- Do not select the same product\_number more than once.

- Do not include prose, Markdown fences, or any additional field.
\end{quote}

\noindent\emph{User message.}
\begin{quote}
\small\ttfamily\raggedright
Available items:

[PRODUCT RECORDS]

\medskip
Pick the best 5 from this list (catalog numbers: [OFFERED PRODUCT NUMBERS])
\end{quote}

\paragraph{Provider request settings.} As shown above, all three providers receive the same logical prompt and are constrained to the
same JSON response structure: a single object containing a
\texttt{top\_5} array of exactly five entries, each containing only
\texttt{rank} and \texttt{product\_number}. OpenAI uses temperature \(1.0\)
and a maximum output allowance of \(4{,}096\) tokens; Anthropic uses
temperature \(1.0\) and a maximum output allowance of \(4{,}096\) tokens.
Gemini uses a \(4{,}096\)-token maximum output allowance, JSON response mode,
and minimal thinking; temperature, top-\(p\), top-\(k\), and candidate count
are left unset.
\par}

\section{Additional Material for Section~\ref{sec:numerics}}
\label{app:additional-numerics}

\subsection{Preference Alignment Across Capacity Configurations}
\label{app:alignment-full-grid}

Here we complement the preference-alignment analysis of
Section~\ref{sec:numerics-alignment} by reporting the complete set of tested
capacity configurations. Figure~\ref{fig:alignment-full-grid} displays mean
revenue loss as a function of \(\rho\), separately for direct-MNL and the
deterministic plug-in, for all seven \((B,K)\) configurations.

\begin{figure}[t]
\centering
\includegraphics[width=0.9\linewidth]
{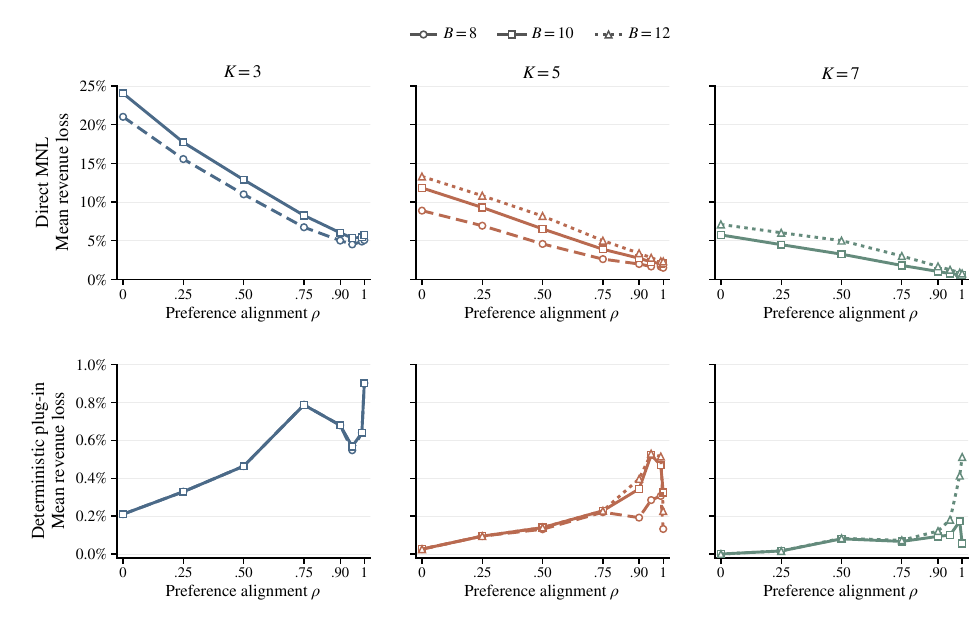}
\caption{Revenue loss by customer--AI preference alignment across the full
capacity grid.}
\label{fig:alignment-full-grid}
\end{figure}

\subsection{Preference Alignment and Filtering Mechanisms}
\label{app:alignment-filtering-mechanisms}

Here we complement the preference-alignment analysis of
Section~\ref{sec:numerics-alignment} by reporting the filtering mechanisms
across the complete set of tested capacity configurations.
Figure~\ref{fig:alignment-filtering-mechanisms-full-grid} reports the mean
difference \(\lvert S_{\mathrm{PL}}^\star\rvert-K\) and the active-filtering
share for all seven \((B,K)\) configurations.

\begin{figure}[t]
\centering
\includegraphics[width=0.9\linewidth]
{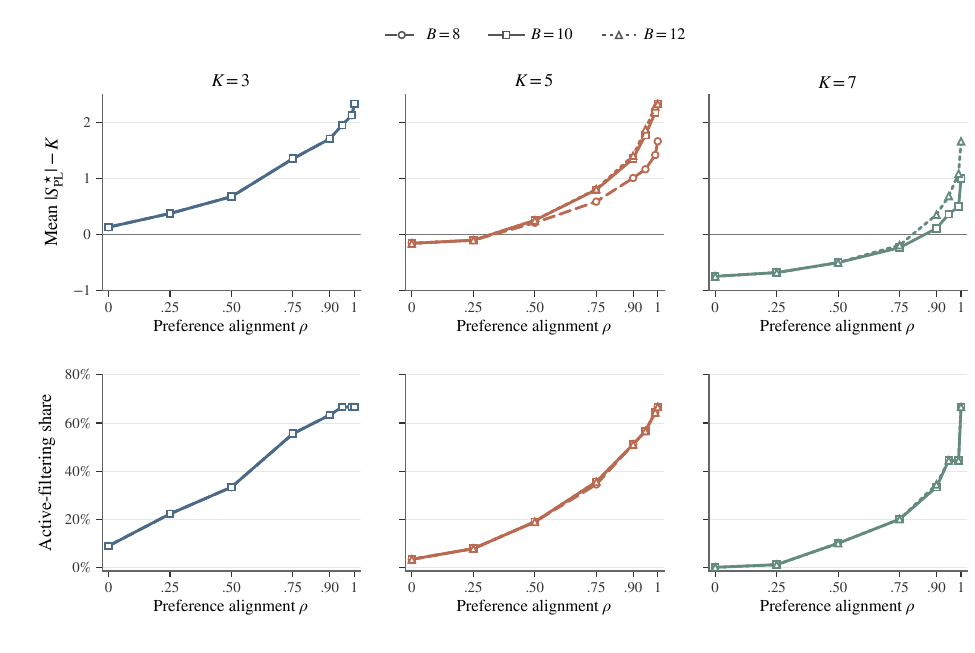}
\caption{Preference alignment and filtering mechanisms across the full
capacity grid.}
\label{fig:alignment-filtering-mechanisms-full-grid}
\end{figure}

\section{Additional Proofs for Deterministic Ranking Mediation}\label{app:deterministic-proofs}

\subsection{Proof of Theorem~\ref{thm:det-hardness}}\label{app:det-hardness}
\paragraph{Assortment planning under ranking preferences.}
We reduce from the assortment planning problem under ranking preferences studied by \citet{aouad2018approximability}. Consider a set of products \(N=\{1,\ldots,n\}\), where each product \(i\in N\) generates revenue \(\bar r_i>0\) when purchased, and a set of customer segments \(\mathcal T\). Segment \(t\in\mathcal T\) occurs with probability \(\lambda_t>0\), where \(\sum_{t\in\mathcal T}\lambda_t=1\), and is associated with an ordered preference list \(L_t=(\ell_{t,1},\ldots,\ell_{t,|L_t|})\) consisting of a subset of \(N\). The products in \(L_t\) are ordered from the most preferred to the least preferred product that segment \(t\) is willing to purchase. Given an assortment \(S\subseteq N\), if \(S\cap L_t\neq\emptyset\), let \(q_t(S)=\min\{q\in\{1,\ldots,|L_t|\}:\ell_{t,q}\in S\}\). Segment \(t\) then purchases product \(\ell_{t,q_t(S)}\), generating revenue \(\bar r_{\ell_{t,q_t(S)}}\). If \(S\cap L_t=\emptyset\), the segment makes no purchase and generates zero revenue.
Thus, for every segment \(t\in\mathcal T\), define the revenue generated by assortment \(S\subseteq N\) as
\[
R_t^{\mathrm{RP}}(S)
=
\begin{cases}
\bar r_{\ell_{t,q_t(S)}}, & S\cap L_t\neq\emptyset,\\
0, & S\cap L_t=\emptyset.
\end{cases}
\]
The expected revenue of assortment \(S\) is then
\[
\Rev^{\mathrm{RP}}(S)
=
\sum_{t\in\mathcal T}\lambda_t R_t^{\mathrm{RP}}(S),
\]
and the objective is to maximize \(\Rev^{\mathrm{RP}}(S)\) over all \(S\subseteq N\).
    For every fixed \(\epsilon>0\), this problem is NP-hard to approximate within factor \(O(n^{1-\epsilon})\) \citep[Theorem~2.1]{aouad2018approximability}.

\paragraph{Instance construction.}
Consider an arbitrary instance of assortment planning under ranking preferences as defined above. We construct an instance of \ref{prob:d-ai-ao} with the same product set \(N\), the same set of segments \(\mathcal T\), and the same segment probabilities \((\lambda_t)_{t\in\mathcal T}\). We set \(K=1\) and \(B=n\). For every product \(i\in N\), we set its revenue in the constructed instance to \(r_i=2\bar r_i\).

For every segment \(t\in\mathcal T\), we construct the deterministic AI ranking \(\sigma_t\) by first listing the products in \(L_t\) according to their order in the preference list, and then listing the products in \(N\setminus L_t\) in increasing order of their product indices. Thus, every product in \(L_t\) is ranked above every product outside \(L_t\), while the relative ordering of the products in \(L_t\) is preserved. Finally, for every product \(i\in N\) and segment \(t\in\mathcal T\), we define the customer-side MNL weight by
\[
v_{it}
=
\begin{cases}
1, & i\in L_t,\\
0, & i\notin L_t.
\end{cases}
\]
The construction clearly takes polynomial time.

The following claim establishes that the reduction preserves the objective function exactly.

\begin{claim}
\label{cl:det-hardness-revenue-preservation}
For every assortment \(S\subseteq N\), we have
\(
\Rev^{\mathrm D}(S)=\Rev^{\mathrm{RP}}(S).
\)
\end{claim}

Before proving the claim, we show how it establishes the theorem. Since \(B=n\), every assortment \(S\subseteq N\) is feasible in the constructed instance. Claim~\ref{cl:det-hardness-revenue-preservation} therefore implies that the two optimization problems have the same objective value for every assortment, and in particular the same optimal value. Moreover, any \(\rho\)-approximate assortment for the constructed instance is a \(\rho\)-approximate assortment for the original ranking-preferences instance. Hence, for every fixed \(\epsilon>0\), a polynomial-time \(O(n^{1-\epsilon})\)-approximation for \ref{prob:d-ai-ao} would yield a polynomial-time \(O(n^{1-\epsilon})\)-approximation for assortment planning under ranking preferences, contradicting \citet[Theorem~2.1]{aouad2018approximability}. It remains to prove Claim~\ref{cl:det-hardness-revenue-preservation}.

\begin{proof}[Proof of Claim~\ref{cl:det-hardness-revenue-preservation}]
Fix an assortment \(S\subseteq N\) and a segment \(t\in\mathcal T\). Suppose first that \(S\cap L_t\neq\emptyset\), and let \(q_t(S)\) be defined as above. Since no product preceding \(\ell_{t,q_t(S)}\) in \(L_t\) belongs to \(S\), and every product outside \(L_t\) is ranked below every product in \(L_t\) under \(\sigma_t\), product \(\ell_{t,q_t(S)}\) is the highest-ranked product in \(S\) under \(\sigma_t\). Since \(K=1\), it is therefore the unique product displayed to segment \(t\). Its customer-side weight is one, and hence
\[
R_t\!\left(\{\ell_{t,q_t(S)}\}\right)
=
\frac{r_{\ell_{t,q_t(S)}}}{2}
=
\bar r_{\ell_{t,q_t(S)}}.
\]
This is exactly the revenue generated by segment \(t\) in the ranking-preferences instance.

Suppose next that \(S\cap L_t=\emptyset\). If \(S=\emptyset\), both models generate zero revenue from segment \(t\). Otherwise, the AI displays some product \(i\in S\). Since \(S\cap L_t=\emptyset\), we have \(i\notin L_t\), and therefore \(v_{it}=0\). Consequently, \(R_t(\{i\})=0\), which again coincides with the revenue generated by segment \(t\) in the ranking-preferences instance.

Thus, the segment-level revenues coincide for every segment \(t\in\mathcal T\). Multiplying by the segment probabilities and summing over \(t\in\mathcal T\) gives
\(
\Rev^{\mathrm D}(S)=\Rev^{\mathrm{RP}}(S).
\)
\end{proof}

\subsection{Additional Proofs for Section~\ref{sec:det-fptas}}
\label{app:det-fptas}

\subsubsection{Proof of Lemma~\ref{lem:det-fptas-thresholds}}\label{app:det-fptas-thresholds}
First, fix an assortment \(S\in\mathcal X_{\boldtau^*}\) and a segment
\(t\in\mathcal T\). Since \(M\subseteq S\), we have
\(\tau_t^*\in S\). Moreover, constraint~(iii) gives
\(
\sum_{i\in S}h_{it}=K.
\)
By definition, \(h_{it}=1\) if and only if product \(i\) is ranked at
least as high as \(\tau_t^*\) under \(\sigma_t\). Hence, exactly \(K\)
products in \(S\) are ranked at least as high as \(\tau_t^*\). These products are precisely the
\(K\) highest-ranked products in \(S\). Therefore,
\(
C_t^{\mathrm D}(S)
=
\{i\in S:h_{it}=1\}.
\)
It follows that
\[
R_t\!\left(C_t^{\mathrm D}(S)\right)
=
\frac{\sum_{i\in S}r_i v_{it}h_{it}}
     {1+\sum_{i\in S}v_{it}h_{it}}.
\]
Multiplying by \(\lambda_t\) and summing over
\(t\in\mathcal T\) yields
\(
\Rev^{\mathrm D}(S)
=
\Rev^{\boldtau^*}(S).
\)

We next show that \(S^*\in\mathcal X_{\boldtau^*}\).
By construction, \(\tau_t^*\in C_t^{\mathrm D}(S^*)\subseteq S^*\)
for every \(t\in\mathcal T\), and therefore \(M\subseteq S^*\).
Moreover, \(S^*\in\mathcal F\), so \(|S^*|\le B\). Since we are in the
active regime, \(|S^*|>K\), and thus
\(\lvert C_t^{\mathrm D}(S^*)\rvert=K\) for every
\(t\in\mathcal T\). Because \(\tau_t^*\) is the lowest-ranked product
in \(C_t^{\mathrm D}(S^*)\), we have
\(
C_t^{\mathrm D}(S^*)
=
\{i\in S^*:h_{it}=1\},
\)
and consequently,
\(
\sum_{i\in S^*}h_{it}=K
\)
for every $t\in\mathcal T$

It remains only to show that \(S^*\subseteq A\). Suppose, toward a
contradiction, that some product \(i\in S^*\setminus A\) exists. By
the definition of \(A\), we have \(h_{it}=0\) for every
\(t\in\mathcal T\). Thus, product \(i\) is ranked below
\(\tau_t^*\) for every segment \(t\), and hence is not contained in
\(C_t^{\mathrm D}(S^*)\) for any segment. Removing \(i\) therefore
does not alter any displayed set:
\(
C_t^{\mathrm D}(S^*\setminus\{i\})
=
C_t^{\mathrm D}(S^*)
\)
for every \(t\in\mathcal T.
\)
Consequently,
\[
\Rev^{\mathrm D}(S^*\setminus\{i\})
=
\Rev^{\mathrm D}(S^*).
\]
Since \(S^*\setminus\{i\}\in\mathcal F\), this gives an optimal
assortment with strictly smaller cardinality than \(S^*\), contradicting
the choice of \(S^*\) as a minimum-cardinality optimal assortment.
Therefore, \(S^*\subseteq A\).
We have shown that \(S^*\) satisfies constraints~(i)--(iii), and hence
\(S^*\in\mathcal X_{\boldtau^*}\).

\subsubsection{Proof of Lemma~\ref{lem:det-fptas-knapsack}}
\label{app:det-fptas-knapsack}
Let
\[
P=C-\sum_{i\in M}c_i
+mK-\sum_{t\in\mathcal T}\sum_{i\in M}h_{it}.
\]
If the exact problem is feasible, restricting a feasible
solution to \(I\) yields a feasible solution to
\eqref{eq:det-fptas-knapsack} with profit exactly \(P\).
Conversely, let \(x\in\{0,1\}^{I}\) be a feasible solution
to \eqref{eq:det-fptas-knapsack} with profit at least \(P\).
On the one hand,
\begin{equation*}
\sum_{t\in\mathcal T}
\left(K-\sum_{i\in M}h_{it}
-\sum_{i\in I}h_{it}x_i\right)
+
\left(C-\sum_{i\in M}c_i-\sum_{i\in I}c_i x_i\right)
\,\ge\,0,
\end{equation*}
since each of the \(m+1\) summands is nonnegative by
feasibility. On the other hand,
\begin{equation*}
\sum_{t\in\mathcal T}
\left(K-\sum_{i\in M}h_{it}
-\sum_{i\in I}h_{it}x_i\right)
+
\left(C-\sum_{i\in M}c_i-\sum_{i\in I}c_i x_i\right)
=
P-\sum_{i\in I}
\left(c_i+\sum_{t\in\mathcal T}h_{it}\right)x_i
\le0,
\end{equation*}
because the profit of \(x\) is at least \(P\).
Thus, this sum equals zero. Since every summand is
nonnegative, each must equal zero. It follows that
\(
\sum_{i\in I}c_i x_i=C-\sum_{i\in M}c_i,
\)
and, for every \(t\in\mathcal T\), we have
\(
\sum_{i\in I}h_{it}x_i
=K-\sum_{i\in M}h_{it}.
\)
Moreover, \(\sum_{i\in I}x_i\le B-|M|\) by feasibility.
Extending \(x\) by setting \(x_i=1\) for every \(i\in M\)
therefore yields a feasible solution to the exact problem.

\end{document}